\documentclass[12pt]{article}
\usepackage[margin=1.25in]{geometry}
\usepackage[T1]{fontenc}
\usepackage{times}

\usepackage{personal}

\title{How Fine-Tuning Allows for Effective Meta-Learning}
\author{Kurtland Chua\thanks{Princeton University. Correspondence to: Kurtland Chua \texttt{<kchua@princeton.edu>}} \quad Qi Lei\footnotemark[1] \quad Jason D. Lee\footnotemark[1]}

\usepackage{natbib}
\setcitestyle{authoryear,open={(},close={)}}
\usepackage{subcaption}

\newcommand{\source}{\mathrm{S}}
\newcommand{\target}{\mathrm{T}}
\newcommand{\meanloss}{\mathcal{L}}

\newcommand{\pgd}{\mathrm{PGD}}
\newcommand{\constraintset}{\mathcal{C}}
\newcommand{\adaptset}{\mathcal{A}}

\DeclareMathOperator{\diam}{diam}
\newcommand{\frzrep}{\textsc{FrozenRep}}
\newcommand{\adprep}{\textsc{AdaptRep}}

\begin{document}
  \maketitle

  \begin{abstract}
  Representation learning has been widely studied in the context of meta-learning, enabling rapid learning of new tasks through shared representations. Recent works such as MAML have explored using fine-tuning-based metrics, which measure the ease by which fine-tuning can achieve good performance, as proxies for obtaining representations. We present a theoretical framework for analyzing representations derived from a MAML-like algorithm, assuming the available tasks use \emph{approximately} the same underlying representation. We then provide risk bounds on the best predictor found by fine-tuning via gradient descent, demonstrating that the algorithm can provably leverage the shared structure. The upper bound applies to general function classes, which we demonstrate by instantiating the guarantees of our framework in the logistic regression and neural network settings. In contrast, we establish the existence of settings where \emph{any algorithm}, using a representation trained with no consideration for task-specific fine-tuning, performs as well as a learner with no access to source tasks in the worst case. This separation result underscores the benefit of fine-tuning-based methods, such as MAML, over methods with “frozen representation” objectives in few-shot learning.
\end{abstract}

  \section{Introduction}
    Meta-learning \citep{thrun2012learning} has emerged as an essential tool
for quickly adapting prior knowledge to a new task with limited data and computational power.
In this context, a meta-learner has access to some different but related source tasks from a shared environment.
The learner aims to uncover some inductive bias from the source tasks to reduce sample and computational complexity for learning a new task from the same environment.
One of the most promising methods is representation learning \citep{bengio2013representation}, \emph{i.e.}, learning a feature extractor (or a common representation) from the source tasks.
At test time, a learner quickly adapts to the new task by fine-tuning the representation and retraining the final layer(s) (see, \emph{e.g.}, prototype networks \citep{snell2017prototypical}).
Substantial improvement over directly learning from a single task is expected in few-shot learning \citep{antoniou2018train}, a setting that naturally arises in many applications including reinforcement learning \citep{mendonca2019guided,finn2017model}, computer vision \citep{nichol2018first}, federated learning \citep{mcmahan2017communication} and robotics \citep{al2017continuous}.

The empirical success of representation learning has led to an increased interest in theoretical analyses of underlying phenomena.
Recent work assumes an explicitly shared representation across tasks \citep{du2020few,tripuraneni2020theory,tripuraneni2020provable,saunshi2020sample,balcan2015efficient}.
For instance, \citet{du2020few} shows a generalization risk bound consisting of an \emph{irreducible} representation error and estimation error.
Without fine-tuning the whole network, the representation error accumulates to the target task and is irreducible even with infinite target labeled samples.
Due to the lack of (representation) fine-tuning during training, we refer to these methods as making use of ``frozen representation'' objectives.
This result is consistent with empirical findings, which suggest substantial performance gains associated with fine-tuning the whole network, compared to just learning the final linear layer \citep{chen2020simple,salman2020adversarially}.

Requiring tasks to be linearly separable on the same features is also unrealistic for transferring knowledge to other domains (\emph{e.g.}, from ImageNet to medical images \citep{raghu2019transfusion}).
Therefore, we consider a more realistic setting, where the available tasks only approximately share the same representation.
We propose a theoretical framework for analyzing the sample complexity of fine-tuning using a representation derived from a MAML-like algorithm.
We show that fine-tuning quickly adapts to new tasks, requiring fewer samples in certain cases compared to methods using ``frozen representation'' objectives (as studied in \citet{du2020few}, and will be formalized in \secref{sec:problem-setting}).
To the best of our knowledge, no prior studies exist beyond fine-tuning a linear model \citep{denevi2018incremental,konobeev2020optimality,collins2020does,lee2020predicting} or only the task-specific layers \citep{du2020few,tripuraneni2020theory,tripuraneni2020provable,mu2020gradients}.
In particular, our work can be viewed as a continuation of the work presented in \citet{tripuraneni2020theory}, where the authors have acknowledged that the framework does not incorporate representation fine-tuning, and thus is a promising line of future work.

The following outlines this paper and its contributions:

\begin{itemize}
	\item In \secref{sec:general}, we outline the general setting and overall assumptions.

	\item In \secref{sec:linear}, we provide an in-depth analysis of the ($d$-dimensional) linear representation setting.
		First, we show that our MAML-like algorithm achieves a rate of $r_{source} = O(\frac{kd}{n_\source T} + \frac{k}{n_\source} + \delta_0\sqrt{\frac{\tr{\Sigma}}{n_\source}})$ on the source tasks and $O(\frac{k}{n_\target} + \delta_0\sqrt{\frac{\tr{\Sigma}}{n_\target}} + r_{source})$ on the target when the norm of the representation change is bounded by $\delta_0$, showing that fine-tuning can handle the approximately-shared representation setting.
		In contrast, “frozen representation” methods have a minimax rate of $\Omega(d/n_\target)$ on the target task under certain task distributions, demonstrating that prior methods can fail.

	\item In \secref{sec:nonlinear}, we extend the analysis to general function classes.
		Our result provides bounds of the form
		\[
			\epsilon_{\mathrm{OPT}} + \epsilon_{\mathrm{EST}} + \epsilon_{\mathrm{REPR}},
		\]
		where $\epsilon_{\mathrm{OPT}}$ is the optimization error, $\epsilon_{\mathrm{EST}}$ is the estimation error, and $\epsilon_{\mathrm{REPR}}$ stands for representation error.

		The \textbf{optimization error} naturally arises when the output of optimization procedure reaches an $\epsilon_{\mathrm{OPT}}$-approximate minimum.
		To control $\epsilon_{\mathrm{OPT}}$, our analysis accounts for nonconvexity introduced by representation fine-tuning, which is presented as a self-contained result in \secref{sec:pgd-performance-bound}.

		The \textbf{estimation error} $\epsilon_{\mathrm{EST}}$ stems from  fine-tuning the target parameters with finite $n_\target$ samples.
		It therefore scales with $1/\sqrt{n_\target}$ and is controlled by the (Rademacher) complexity of the target fine-tuning set.

		Finally, the \textbf{representation error} $\epsilon_{\mathrm{REPR}}$ is the error associated with learning the centered representation while adapting to the task-specific deviations. Therefore it naturally consists of two terms: one term scales with $1/\sqrt{n_\source T}$ on learning the centered representation jointly with all $T$ source tasks; and the other term scales with $1/\sqrt{n_\source}$ on the small adaptation to each source task.

	\item In \secref{sec:case-studies}, we instantiate our guarantees in the logistic regression and two-layer neural network settings.

	\item In \secref{sec:gen-hard-case}, we extend the separation result presented in \secref{sec:linear} to a non-linear setting.

	\item In \secref{sec:simulations}, we experimentally verify the separation result in the linear setting from \secref{sec:linear}.
\end{itemize}

\subsection{Related Work}

	The empirical success of MAML \citep{finn2017model} or, more generally, meta-learning, invokes further theoretical analysis from both statistical and optimization perspectives.
	A flurry of work engages in developing more efficient and theoretically-sound optimization algorithms \citep{antoniou2018train,nichol2018first,li2017meta} or showing convergence analysis \citep{fallah2020convergence,zhou2019efficient,rajeswaran2019meta,collins2020task}.
	Inspired by MAML, a line of gradient-based meta-learning algorithms have been widely used in practice \cite{nichol2018first,al2017continuous,jerfel2018online}.
	Much follow-up work focused on online-setting with regret bounds \citep{denevi2018incremental,finn2019online,khodak2019adaptive,balcan2015efficient,alquier2017regret,bullins2019generalize,pentina2014pac}.

	The statistical analysis of meta-learning can be traced back to \citet{baxter2000model, maurer2005algorithmic}, with a principle of inductive bias learning.
	Following the same setting, \citet{amit2018meta, konobeev2020optimality, maurer2016benefit, pentina2014pac} assume a shared meta distribution for sampling the source tasks and measure generalization error/gap averaged over the meta distribution.
	Another line of work connects the target performance to source data with some distance measure between distributions \citep{ben2008notion,ben2010theory,mohri2012new}.
	Finally, another series of works studied the benefits of using additional ``side information'' provided with a task for specializing the parameters of an inner algorithm \citep{denevi2020advantage,denevi2021conditional}.

	The hardness of meta-learning also attracts some investigation under various settings. Some recent work studies the meta-learning performance in worst-case setting \citep{collins2020task,hanneke2020no,hanneke2020value,kpotufe2018marginal,lucas2020theoretical}.
	\citet{hanneke2020no} provides a no-free-lunch result with problem-independent minimax lower bound, and \citet{konobeev2020optimality} also provide problem-dependent lower bound on a simple linear setting.

  \section{General Setting}
  \label{sec:general}
    \subsection{Notation}

  Let $[n] \defas \set{1, \dots, n}$.
  We denote the vector $L_2$-norm as $\norm[2]{\cdot}$, and the matrix Frobenius norm as $\norm[F]{\cdot}$.
  Additionally, $\inprod{\cdot, \cdot}$ can denote either the Euclidean inner product or the Frobenius inner product between matrices.

  For a matrix $A$, we let $\sigma_i(A)$ denote its $i^{\text{th}}$ largest singular value.
  Additionally, for positive semidefinite $A$, we write $\lambda_{\max}(A)$ and $\lambda_{\min}(A)$ for its largest and smallest eigenvalues, and $A^{1/2}$ for its principal square root.
  We write $P_A$ for the projection onto the column span of $A$, denoted $\col{A}$, and $P_A^\perp \defas I - P_A$ for the projection onto its complement.

  We use standard $O, \Theta$, and $\Omega$ notation to denote orders of growth.
  We also use $a \lesssim b$ or $a \ll b$ to indicate that $a = O(b)$.
  Finally, we write $a \asymp b$ for $a = \Theta(b)$.

\subsection{Problem Setting}
\label{sec:problem-setting}

  We assume that the learner has access to $T$ source tasks.
  Every task $t \in [T]$ is associated with a distribution $\mu_t$ over the set $\mathcal{X} \times \mathcal{Y}$ of input-label pairs.
  We consider $\mathcal{X} \subseteq \reals^d$ and $\mathcal{Y} \subseteq \reals$ throughout.
  For every task, the learner observes $n_\source$ i.i.d. samples $\set{(x_{i, t}, y_{i, t})}_{i \in [n_\source]}$ from $\mu_t$.
  We collect these inputs into matrices $X_t \in \reals^{n_\source \times d}$ and labels into vectors $y_t \in \reals^{n_\source}$ for convenience.
  Finally, we measure learner performance using the loss function $\ell: \reals \times \mathcal{Y} \to \reals$.

  We aim to find a common structure across the source tasks that could be leveraged for learning future tasks.
  As done by \citet{du2020few}, we study the case where the common structure is given by a \emph{representation}, a mapping $\phi: \mathcal{X} \to \mathcal{Z}$ from the input space to a latent space $\mathcal{Z} \subseteq \reals^k$.
  We assume that these representations lie in a function class $\mathcal{F}$ parameterized by some normed space.
  In contrast to prior work however, we \emph{do not} assume the existence of a common representation -- instead, each source task has an associated representation $\phi_{\theta_t}$ close to a fixed representation $\phi_{\theta_0}$, \textit{i.e.} $\norm{\theta_t - \theta_0}$ is small.
  As before, we apply a task-specific linear transformation to the representation to obtain the final prediction.
  Formally, the predictor for task $t \in [T]$ is given by $x \mapsto \inprod{w_t, \phi_t(x)}$.
  We then consider the following optimization problem for some $\delta_0 > 0$:
  \begin{equation}
    \label{eqn:adaptation-source-procedure}
    \min_{\theta_0}\min_{\substack{\theta_t, w_t \\ \norm{\theta_t - \theta_0} \leq \delta_0}}\frac{1}{n_\source T}\sum_{t=1}^{T}\sum_{i=1}^{n_\source}\ell(\inprod{w_t, \phi_{\theta_t}(x_{i, t})}, y_{i, t}).
  \end{equation}
  We will refer to the procedure above as \adprep{}, as it can be intuitively described as finding an initialization $\phi$ in representation space such that there exists a good representation nearby for every task (and thus a learner simply needs a small \emph{adaptation/fine-tuning} step to achieve good performance).
  We note that the objective can be viewed as a constrained form of algorithms found in the literature such as iMAML \citep{rajeswaran2019meta} and Meta-MinibatchProx \citep{zhou2019efficient}.
  However, we do not have a train-validation split, as is widespread in practice.
  This setup is motivated by results in \citet{bai2020important}, which show that data splitting may not be preferable performance-wise, assuming realizability.
  Furthermore, empirical evaluations have demonstrated successes despite the lack of such a split \citep{zhou2019efficient}.

  Let $\theta_0$ be the initialization obtained from \eqnref{eqn:adaptation-source-procedure}, which we will now use to learn future tasks.
  More concretely, let $\mu_{T + 1}$ be a distribution over $\mathcal{X} \times \mathcal{Y}$ from which we have $n_\target$ i.i.d. samples $\set{(x_{i}, y_{i})}_{i \in [n_\target]}$.
  As before, we collect these samples into a matrix $X \in \reals^{n_\target \times d}$ and a vector $y \in \reals^{n_\target}$.
  We then adapt to the target task by solving
  \begin{equation}
    \label{eqn:adaptation-target-procedure}
    \min_{\substack{\theta, w \\ \norm{\theta - \theta_0} \leq \delta_0}}\frac{1}{2n_\target}\sum_{i=1}^{n_\target}\ell(\inprod{w, \phi_\theta(x_{i})}, y_{i}).
  \end{equation}
  In the following sections, we will analyze the performance of the learned predictor for the population loss:
  \[
  \meanloss_{\mu_{T + 1}}(\theta, w) \defas \frac{1}{2}\expt[(x, y) \sim \mu_{T + 1}]{\ell(\inprod{w, \phi_\theta(x)}, y)}.
  \]
  We focus on the few-shot learning setting for the target task, where we assume limited access to target data, and a learner needs to effectively use the source tasks to learn the target task quickly.

  For \adprep{} to be sensible, we need to ensure the existence of a desirable initialization.
  We do so by assuming that there exists an initialization $\theta_0^\ast$ such that for any $t \in [T]$, there exists a representation $\theta_t^\ast$ with $\norm{\theta_t^\ast - \theta_0^\ast} \leq \delta_0$ and a predictor $w_t^\ast$ so that $\mu_t$ is given by
  \[
  x \distas p_t, \quad y \conditionedon x \distas q(\cdot | \inprod{w_t^\ast, \phi_{\theta_t}^\ast(x)})
  \]
  with distribution $q$ to model some addictive noise.
  Specifically, in regression, we set $\expt{y_t \suchthat x} = \inprod{w_t^\ast, \phi_{\theta_t}^\ast(x)}$.
  With an appropriate choice of loss function $\ell$, we can guarantee that the optimal predictor under the population loss is $x \mapsto \inprod{w_t^\ast, \phi_t^\ast(x)}$.

  Throughout the paper, we assume the existence of an oracle during source training time for solving \eqnref{eqn:adaptation-source-procedure}, as in \citet{du2020few,tripuraneni2020theory}.
  For detailed analyses of source training optimization, we refer the reader to \citet{ji2020convergence,wang2020global}.
  Nevertheless, representation fine-tuning introduces nonconvexity during target time not present in prior work, where one only needed to solve the convex problem of optimizing a final linear layer.
  Thus, our bounds explicitly take into account optimization performance on \eqnref{eqn:adaptation-target-procedure}.
  To this end, we analyze the use of projected gradient descent (PGD), which applies to a wide variety of settings, under certain loss landscape assumptions.
  These standalone results are also provided in \secref{sec:pgd-performance-bound}.

  As a point of comparison with \adprep, we will also be analyzing the ``frozen representation'' objective used in \citet{du2020few,tripuraneni2020theory}.
  In particular, using the notation introduced above, such objectives consider the following optimization problem:
  \begin{equation}
    \label{eqn:frz-rep-objective}
    \min_{\theta_0}\min_{w_t}\frac{1}{n_\source T}\sum_{t=1}^{T}\sum_{i=1}^{n_\source}\ell(\inprod{w_t, \phi_{\theta_0}(x_{i, t})}, y_{i, t}).
  \end{equation}
  Due to the fact that the representation, once chosen, is fixed/frozen for all source tasks, we refer to the representation learning method above as \frzrep{} throughout the rest of the paper.
  In Sections \ref{sec:erm-hard-case} and \ref{sec:gen-hard-case}, we will demonstrate that unlike with \adprep{}, there exists cases where \frzrep{} is unable to take advantage of the fact that the tasks approximately share representations.

  \section{\adprep{} in the Linear Setting}
    \label{sec:linear}
To illustrate the theory, we first examine \adprep{} in the linear setting.
That is, we consider the set of linear transformations $\reals^d \to \reals^k$ for $d > k$, equipped with the Frobenius norm, so that $\phi_B(x) = B^\top x$.
In this setting, we provide a performance bound for \adprep{}, and then exhibit a specific construction for which the baseline method in \citet{du2020few} fails to find useful representations.

\subsection{Statistical Assumptions}
\label{sec:lin-data-assumptions}

	We proceed to instantiate the data assumptions outlined in \secref{sec:problem-setting}.
	First, we assume that the inputs for all tasks come from a common zero-mean distribution $p$, with covariance $\expt[x \sim p]{xx^\top} = \Sigma$.
	Let $\kappa = \lambda_{\max}(\Sigma)/\lambda_{\min}(\Sigma)$ be the condition number of this covariance matrix.
	As done by \citet{du2020few}, we impose the following tail condition:

	\begin{assumption}[Sub-Gaussian input]
		\label{assump:sg-inputs}
		There exists $\rho > 0$ such that if $x \sim p$, then $\Sigma^{-1/2}x$ is $\rho^2$-sub-Gaussian\footnote{Recall that a zero-mean random vector $v$ is $\rho^2$-sub-Gaussian if for any fixed unit vector $u$, $\expt{\exp(\lambda v^\top u)} \leq \lambda^2\rho^2/2$.}.
	\end{assumption}

	This assumption is used in the proofs to guarantee probabilistic tail bounds, and can be replaced with other conditions with appropriate modifications to the analysis.
	Finally, we define $q\left(\cdot \given \mu\right) \sim \gaussian{\mu, \sigma^2}$ for a fixed $\sigma > 0$.

	We proceed to parametrize the tasks.
	Let $B^\ast$ be the ground truth initialization point, and $\Delta_t^\ast$ be the task-specific fine-tuning for task $t$, where $\norm[F]{\Delta_t^\ast} \leq \delta_0$.
	Furthermore, let $w_t^\ast \in \reals^k$ be the task-specific predictor weights, which we combine into a matrix $W^\ast = [w_1^\ast, \dots, w_T^\ast] \in \reals^{k \times T}$.

	\begin{assumption}[Source task diversity]
		\label{assump:lin-diversity-cond}
		For any $t \in [T]$, $\norm[2]{w_t^\ast} = \Theta(1)$, and $\sigma_k^2(W^\ast) = \Omega(T/k)$.
	\end{assumption}

	Since the predictor weights all have $\Theta(1)$ norm, and thus $\sum_{i \in [k]}\sigma_i(W^\ast)^2 = \norm[F]{W^\ast}^2 = \Theta(T)$, the assumption on $\sigma_k(W^\ast)^2$ implies that the weights covers directions in $\reals^k$ roughly evenly.
	This condition is satisfied with high probability when the $w^\ast_t$ are sampled from a sub-Gaussian distribution with well-conditioned covariance.

	Finally, we evaluate the performance of the learner on a target task $\theta^\ast \defas B^\ast w^\ast + \delta^\ast$ for some $w^\ast$ and $\norm[2]{\delta^\ast} \leq \delta_0$.

	To gain an intuition for the parameters, note that $(B^\ast + \Delta_t^\ast)w_t^\ast = B^\ast w_t^\ast + \delta_t^\ast$, where $\norm[2]{\delta_t^\ast} \lesssim \delta_0$ by the norm conditions in \assumpref{assump:lin-diversity-cond}.
	Thus, we can think of the assumptions on the source predictor weights as ensuring their proximity to a rank-$k$ space.

	Finally, as a convention since the parameterization is not unique, we define the optimal parameters so that $(B^\ast)^\top\Sigma\delta_t^\ast=0$ for any $t \in [T]$.
	This results in no loss of generality, as we can always redefine $w_t^\ast$ and $\delta_t^\ast$ as such.

\subsection{Training Procedure}
\label{sec:lin-training-procedure}

	\textbf{(Source training)} We can write the objective \eqnref{eqn:adaptation-source-procedure} in this setting as
	\begin{equation*}
		\min_{B}\min_{\substack{\Delta_t, w_t \\ \norm[F]{\Delta_t} \leq \delta_0}}\frac{1}{2n_\source T}\sum_{t=1}^{T}\norm[2]{y_t - X_t(B + \Delta_t)w_t}^2.
	\end{equation*}
	However, note that the objective does not impose any constraint on the predictor, as we can offload the norm of $\Delta_t$ onto $w_t$.
	Therefore, we instead consider the regularized source training objective
	\begin{equation}
	\label{eqn:linear-source-procedure}
		\min_{B}\min_{\Delta_t, w_t}\frac{1}{2n_\source T}\sum_{t=1}^{T}\norm[2]{y_t - X_t(B + \Delta_t)w_t}^2 + \frac{\lambda}{2}\norm[F]{\Delta_t}^2 + \frac{\gamma}{2}\norm[2]{w_t}^2.
	\end{equation}
	As will be shown in \secref{sec:lin-case-proof}, the regularization is equivalent to regularizing $\sqrt{\lambda\gamma}\norm[2]{\Delta_tw_t}$, which is consistent with the intuition that $\delta_t^\ast$ has small norm.

	\textbf{(Target training)} Letting $B_0$ be the obtained representation after orthonormalization, we adapt to the target task by optimizing
	\begin{equation}
	\label{eqn:lin-targ-obj}
	\mathcal{L}_\beta(\Delta, w) = \frac{1}{2n}\norm[2]{y - \beta X\left(A_{B_0} + \Delta\right)(w_0 + w)}^2,
	\end{equation}
	where $A_{B_0} \defas [B_0 \ B_0] \in \reals^{d \times 2k}$ and $w_0 = [u, -u]$ for some unit-norm vector $u \in \reals^k$.
	To optimize the objective, we perform PGD on \eqnref{eqn:lin-targ-obj} with
	\[
		\constraintset_\beta \defas \set{(\Delta, w) \suchthat \norm[F]{\Delta} \leq c_1/\beta, \norm[2]{w} \leq c_2/\beta}
	\]
	as the feasible set, where we explicitly define $c_1$ and $c_2$ in \secref{sec:lin-case-proof}.

	To understand the choice of training objective above, observe that the predictor can be written as
	\[
	x \mapsto \beta(x^\top A_{B_0}w + x^\top \Delta w_0 + x^\top \Delta w),
	\]
	where the ``antisymmetric'' initialization scheme ensures that $A_{B_0}w_0 = 0$.
	Due to the choice of $\constraintset_\beta$, the first two terms are of norm $O(1)$, while the last term is of norm $O(1/\beta)$.
	Therefore, for large enough $\beta$, we can treat the cross term $\Delta w$ as a negligible perturbation, and the predictor is approximately linear in the parameters.
	Indeed, one can show that $\meanloss_\beta$ is thus approximately convex, guaranteeing that the best solution found by PGD is nearly optimal.

\subsection{Performance Bound}

	Now, we provide a performance bound on the performance of the algorithm proposed in the previous section, which we prove in \secref{sec:lin-case-proof}.
	We define the following rates\footnote{Log factors and non-dominant terms have been suppressed for clarity. Full rates are presented in the appendix.} of interest:
	\begin{align*}
		r_\source(n_\source, T) &\defas \frac{\sigma^2kd}{n_\source T} + \sigma\delta_0\norm[2]{\Sigma}^{1/2}\sqrt{\frac{kd}{n_\source T}} + \frac{\sigma^2 k}{n_\source} + \frac{\sigma\delta_0}{\sqrt{n_\source}}\sqrt{\tr{\Sigma}} \\
		r_{\target}^{(1)}(n_\target) &\defas kr_{\source}(n_\source, T) + \frac{\sigma^2k}{n_\target} + \frac{\sigma\delta_0}{\sqrt{n_\target}}\sqrt{\tr\Sigma} \\
		r_{\target}^{(2)}(n_\target) &\defas \frac{\sigma^2k}{n_\target} + \left[1 + \sqrt{\frac{\kappa}{\lambda_{\min}} kr_{\source}(n_\source, T)}\right]\frac{\sigma\delta_0}{\sqrt{n_\target}}\sqrt{\tr\Sigma}
	\end{align*}
	\includestatement[dir=results/linear/perf]{linMainResult}

	We briefly remark on the three available rates, which corresponds to three different subalgorithms depending on the amount of target data.
	Firstly, as will be proven in \secref{sec:lin-case-proof}, $r_\source$ is a performance bound on the learner during source training.
	Then:

	\begin{enumerate}
		\item In the most data-starved regime, we restrict the learner so that it can only adapt to learn $\delta^\ast$ for the target task to obtain the first rate.
			Notably, $r_\target^{(1)}$ incorporates the resulting irreducible source error in the form of $kr_\source(n_\source, T)$\footnote{When measuring average target performance, $kr_\source(n_\source,T)$ can be replaced by $r_\source(n_\source, T)$; see \secref{sec:logistic-regression}; \citet{du2020few}.}.

		\item The second rate $r_{\target}^{(2)} $ results when we allow the learner to adapt more (requiring more target samples) to reduce the irreducible error due to finite $n_S$.
			Observe that this additional complexity in the fine-tuning step is captured by the multiplicative factor that shrinks to $1$ in the limit of infinite source samples.

		\item Finally, we can obtain the trivial $\sigma^2d/n_\target$ rate by ignoring $B_0$, which matches the minimax lower bound for standard linear regression (see e.g., \citet{duchi2016lecture}).
	\end{enumerate}

\subsection{A Hard Case for \frzrep{}}
\label{sec:erm-hard-case}

	In what follows, we demonstrate the existence of a family of task distributions satisfying the assumptions outlined in \secref{sec:lin-data-assumptions} that is difficult for the method in \citet{du2020few}, which we will refer to as \frzrep{}\footnote{This is in reference to the fact that the representation is frozen during source training, i.e. no task-specific fine-tuning.}.
	More explicitly, we prove an $\Omega(d/n)$ minimax rate on the target task when using an \frzrep{}-derived representation, \emph{even with access to infinite source tasks and data}.
	Since this rate is achievable via training on the target task directly, this demonstrates that \frzrep{} fails to capture the shared information between the tasks.
	In contrast, by specializing \thmref{thm:lin-case-main-result} to the proposed family of tasks, we will show that \adprep{} can indeed achieve a strictly faster statistical rate.

	We proceed to explicitly define the objects in \secref{sec:lin-data-assumptions}.
	Let $p$ be a Gaussian distribution on $\reals^d$ with covariance
	\[
		\Sigma =
		\begin{bmatrix}
		\epsilon I_{d - k} & 0 \\
		0 & I_k \\
		\end{bmatrix}
	\]
	for a fixed $\epsilon \in (0, 1)$.
	We define $E^\ast, E_k\subset \mathbb{R}^{d}$ to be the eigenspaces corresponding to the first and second block of $\Sigma$, respectively, \emph{i.e.}
	\[
		E^\ast = \col{\begin{bmatrix}
			\epsilon I_{d - k} \\
			0 \\
		\end{bmatrix}} \quad \text{and} \quad E_k = \col{\begin{bmatrix}
			0 \\
			I_k \\
		\end{bmatrix}}.
	\]
	Then, for an orthogonal matrix $B \in \reals^{d \times k}$, define a corresponding task distribution given by
	\begin{equation}
	\label{eqn:lin-hard-task-dist}
	\theta = \frac{1}{\sqrt{2\epsilon}}Bw + \delta,
	\end{equation}
	where $w$ and $\delta$ are sampled uniformly at random from the unit spheres in $\reals^k$ and $E_k$, respectively.
	The family of interest is the set of task distributions induced by any $B^\ast$ such that $\col{B^\ast} \subset E^\ast$.

	In this setting, we can write the \frzrep{} objective in \eqnref{eqn:frz-rep-objective} as
	\begin{equation}
		\label{eqn:erm-lin-src-procedure}
		\hat{B} = \argmin_{B}\min_{w_t}\frac{1}{2n_\source T}\sum_{t \in [T]}\norm[2]{y_t - X_tBw_t}^2.
	\end{equation}
	First, we characterize the span of $\hat{B}$ \emph{in the limit of infinite source tasks and data}.
	Intuitively, since both $B^\ast w$ and $\delta$ both lie in (distinct) rank-$k$ spaces, but $(1/\sqrt{2\epsilon})\norm[2]{\Sigma^{1/2}B^\ast w} \leq \norm[2]{\Sigma^{1/2}\delta}$ for any $(w, \delta)$ from the task distribution (and thus the error along $\delta$ is larger), \frzrep{} learns $E_k$ rather than $B^\ast$.

	\includestatement[dir=results/linear/hardcase,showproof=false]{ermLearnsWrongSpace}

	The claim is proven in \secref{sec:erm-hard-case-proofs}.
	Although ``incorrect'', it is unclear \emph{a priori} that this choice of $\hat{B}$ is undesirable performance-wise -- we now show that this indeed the case.
	In fact, \emph{any algorithm} making use of $\hat{B}$, in the worst case, cannot perform any better than a learner that is constrained to only use target data.

	\includestatement[dir=results/linear/hardcase,showproof=false]{ermMinimaxBound}

	The previous result, which we prove in \secref{sec:erm-hard-case-proofs}, shows that \emph{any procedure} making use of target samples to learn a predictor of the form $\hat{B}\hat{w} + \hat\delta$ has a minimax rate of $\Omega(d/n)$.
	This includes target-time fine-tuning procedures used by methods in practice such as iMAML, MetaOptNet \citep{lee2019metalearning}, and R2D2 \citep{bertinetto2019metalearning}.
	Notably, this rate is achievable by performing linear regression solely on the target samples, reflecting that the \frzrep{} learner failed to capture the shared task structure.
	In contrast, by specializing the guarantee of \thmref{thm:lin-case-main-result} to this setting, we have the following result:

	\begin{corollary}
		Fix $k = \Theta(1)$ and $d \gg k$, and set $\epsilon = k / d$.
		Furthermore, assume that $n_\source T \gtrsim d^2$, and $n_\source \geq n_\target \asymp d$.
		Then, for a fixed target task in $S_{B^\ast}$ as defined in \thmref{thm:erm-minimax-bound}, with probability at least $1 - \delta$ over the draw of samples, the procedure outlined in \secref{sec:lin-training-procedure} achieves excess risk bounded as
		\begin{align*}
		 	\expt{(x^\top\theta^\ast - x^\top\hat\theta)^2} \lesssim \min\left(\frac{\sigma}{\sqrt{n_\target}}, \frac{\sigma^2 d}{n_\target}\right).
		\end{align*}
	\end{corollary}

	In constrast to the rate above, the minimax rate in \thmref{thm:erm-minimax-bound} remains bounded away from $0$ when $n_\target \asymp d$.
	Therefore, in addition to performing at least as well as the minimax rate for \frzrep{} for \emph{any} $n_\target \gtrsim d$, there exists a strict separation between the methods that widens as $n_\target \to \infty$.
	Furthermore, the guarantee can be achieved with finite source samples.
	Thus, we have established a setting where incorporating a notion of representation fine-tuning into meta-training is \emph{provably necessary} for the meta-learner to succeed.

  \section{\adprep{} in the Nonlinear Setting}
    \label{sec:nonlinear}

We now describe a general framework for analyzing fine-tuning in general function classes.
To simplify the notation, we modify the setting described in \secref{sec:problem-setting} so that both the representation $\phi$ and task-specific weight vector $w$ are captured by one parameter $\theta \in \Theta$, with corresponding predictor $g_\theta$.

Throughout this section, we denote the population loss induced by a predictor $g$ as
\[
	\meanloss^g_\infty(h):=\expt[(x, y) \sim \mu_g]{\ell(h(x), y)},
\]
where $\mu_g$ samples $x \distas p$ and $y \conditionedon x \distas q(\cdot | g(x))$.
Additionally, we let $\meanloss^g(h)$ denote the corresponding finite-sample quantity\footnote{Note that we have omitted the samples from the notation for brevity.}.

Furthermore, for a fixed parameter $\theta$ and a set $\constraintset$, we define the set $\adaptset^\constraintset_\theta$ to be
\[
	\adaptset_\theta^\constraintset \defas \set{g_{\theta'} \suchthat \theta' - \theta \in \constraintset}.
\]
Intuitively, we can think of $\constraintset$ as the possible ways a learner could adapt, and $\adaptset^\constraintset_\theta$ is the resulting set of possible predictors given an initialization $\theta$.
For convenience, we define $\left(\adaptset_\theta^{\constraintset}\right)^{\otimes T}$ to be the set of functions mapping $\mathcal{X}^T \to \mathcal{Y}^T$ defined as
\[
	\left(\adaptset_\theta^{\constraintset}\right)^{\otimes T} \defas \set{(x_{i})_{i \in [T]} \mapsto (f_i(x_i))_{i \in [T]} \suchthat f_1, \dots, f_T \in \mathcal{A}_{\theta}^\constraintset}.
\]
That is, $\left(\adaptset_\theta^{\constraintset}\right)^{\otimes T}$ can be interpreted as the set of possible choices of $T$ predictors for $T$ source tasks that are all close to some initialization $\theta$.

\subsection{Training Procedure}
\label{sec:gen-train-proc}

	The training procedure is defined by a set of possible initializations $\Theta_0 \subseteq \Theta$, and fine-tuning sets $\constraintset_\source, \constraintset_\target \subseteq \Theta$.
	We then rewrite the objective defined in \eqnref{eqn:adaptation-source-procedure} in terms of $\theta$ as
	\begin{equation}
	\label{eqn:adaptation-src-proc-general}
	\theta_0 = \argmin_{\theta \in \Theta_0}\min_{g_t \in \adaptset_\theta^{\constraintset_\source}}\frac{1}{2n_\source T}\sum_{t=1}^{T}\sum_{i=1}^{n_S}\ell(g_{t}(x_{i, t}), y_{i, t}).
	\end{equation}
	Note that $\constraintset_\source$ is the feasible set of task-specific fine-tuning options for source training.
	Then, given samples from a target task $g^\ast$, we perform $T_\pgd$ steps of PGD with step size $\eta$ on the objective $\delta \mapsto \meanloss^{g^\ast}(g_{\theta_0 + \delta})$, with feasible fine-tuning set $\constraintset_\target$.

\subsection{Assumptions}
\label{sec:gen-setting-assumptions}

	We now outline the assumptions we make in this setting.
	\begin{assumption}
		\label{assump:loss-fn-assumptions}
		For any $y \in \mathcal{Y}$, $\ell(\cdot, y)$ is $1$-Lipschitz\footnote{Observe that this is not restrictive as one can simply rescale the loss, and is assumed for simplicity of presentation.} and convex, and that $\abs{\ell(0, y)} \leq B$.
	\end{assumption}
	To ensure transfer from source to target, we impose the following condition, a specific instance of which was proposed by \citet{du2020few} in the linear setting, and proposed by \citet{tripuraneni2020theory} for general settings:
	\begin{assumption}[Source tasks are $(\nu, \epsilon)$-diverse]
		\label{assump:nu-eps-diversity-def}
		There exists constants $(\nu, \epsilon)$ such that if $\rho$ is the distribution of target tasks, then for any $\theta \in \Theta_0$,
		\begin{align*}
			&\expt[g^\ast \sim \rho]{\inf_{g \in \adaptset_\theta^{\constraintset_\target}}\meanloss_\infty^{g^\ast}(g) - \meanloss_\infty^{g^\ast}(g^\ast)} \leq \frac{1}{\nu}\left[\frac{1}{T}\sum_{t = 1}^T\inf_{g \in \adaptset_\theta^{\constraintset_\source}}\meanloss_\infty^{g_t^\ast}(g) - \meanloss_\infty^{g_t^\ast}(g_t^\ast)\right] + \epsilon.
		\end{align*}
		That is, the average best-case performance on the set of target tasks is controlled by the task-averaged best-case performance on the source tasks.
	\end{assumption}
	The $(\nu, \epsilon)$-diversity assumption ensures that optimizing $\theta$ for the average source task performance results in controlled average target task performance.
	Note that we weakened the condition in \citet{tripuraneni2020theory} to bound the average rather than worst-case target performance, as is more suitable for higher-dimensional settings.

	Finally, we make several assumptions to ensure that PGD finds a solution close in performance to the optimal fine-tuning parameter in $\constraintset_\target$.
	We remark that these assumptions are specific to the choice of fine-tuning algorithm.

	\begin{assumption}[Approximate linearity in fine-tuning]
		\label{assump:adapt-approx-linearity}
		Let $x_1, \dots, x_{n_\target}$ be the set of target inputs.
		Then, there exists $\beta, L$ such that
		\begin{align*}
			\sup_{\delta \in \constraintset_T}\frac{1}{n_\target}\sum_{i = 1}^{n_\target}\norm[2]{\grad[\theta]^2g_{\theta + \delta}(x_i)}^2 &\leq \beta^2 \quad \text{and} \quad \sup_{\theta \in \Theta_0}\frac{1}{n_\target}\sum_{i = 1}^{n_\target}\norm[2]{\grad[\theta]g_{\theta}(x)}^2 \leq L^2.
		\end{align*}
	\end{assumption}

	\begin{assumption}
		\label{assump:target-adapt-norm-bound}
		$\sup_{\delta \in \constraintset_\target}\norm[2]{\delta} \leq R$ for some $R$.
	\end{assumption}

\subsection{Performance Bound}

	Before we present the performance bound, recall that for a set $\mathcal{H}$ of functions $\reals^d \to \reals^k$ on $n$ samples, its Rademacher complexity $\mathcal{R}_n(\mathcal{H})$ on $n$ samples is given by
	\[
		\mathcal{R}_{n}(\mathcal{H}) \defas \expt[\epsilon, X]{\frac{1}{n}\abs{\sup_{h \in \mathcal{H}}\sum_{i = 1}^n\sum_{j = 1}^k\epsilon_{ij}h_j(x_{i})}},
	\]
	where $(\epsilon_{ij})$ are i.i.d. Rademacher random variables and $(x_i)$ are i.i.d. samples from some (preset) distribution.

	\includestatement[dir=results/nonlinear/perf]{genPerformanceBound}

	Note that the Rademacher complexity terms above decays in most settings as
	\begin{align*}
		\mathcal{R}_{n_\target}(\mathcal{A}_{\theta_0}^{\constraintset_\target}) &\asymp O\left(\frac{\diam{\constraintset_\target}}{\sqrt{n_\target}}\right) \\
	\frac{1}{T}	\mathcal{R}_{n_\source}\left[\Union_{\theta \in \Theta_0}\left(\adaptset_{\theta}^{\constraintset_\source}\right)^{\otimes T}\right] &\asymp O\left(\frac{C(\Theta_0)}{\sqrt{n_\source\target}} + \frac{\diam{\constraintset_\source}}{\sqrt{n_\source}}\right),
	\end{align*}
	where $C(\Theta_0)$ represents some complexity measure of $\Theta_0$, and $\diam{\constraintset_\source}$ and $\diam{\constraintset_\target}$ represents a measure of the size of the fine-tuning sets $\constraintset_\source$ and $\constraintset_\target$.
	A full proof is provided in \secref{sec:gen-perf-proofs}.
	Nevertheless, we provide a brief intuition on the result, as well as how the assumptions contribute to the final bound.
	Let $\theta_{\mathrm{opt}}$ denote the best-performing solution found by PGD, $\theta_{\mathrm{ERM}}$ and $\bar\theta$ the ERM and population-level optimal solution in $\adaptset_{\theta_0}^{\constraintset_\target}$ respectively, and $\theta^\ast$ the ground truth predictor.

	\textbf{Optimization error ($\epsilon_{\mathrm{OPT}}$).} The difference in performance between $\theta_{\mathrm{opt}}$ and $\theta_{\mathrm{ERM}}$, \emph{i.e.} the error due to the optimization procedure, is bounded by approximate linearity.

	\textbf{Estimation error ($\epsilon_{\mathrm{EST}}$).} By uniform convergence, $\theta_{\mathrm{ERM}}$ (and thus $\theta_{\mathrm{opt}}$) performs similarly to $\bar\theta$.

	\textbf{Representation/approximation error ($\epsilon_{\mathrm{REPR}}$).} Via the $(\nu, \epsilon)$-diversity condition, $\bar\theta$ is connected to $\theta^\ast$ explicitly (that is, the performance on the target task is controlled by the performance on the source tasks).
	Furthermore, the performance over source tasks can be bounded via uniform convergence.

  \section{Case Studies}
    \label{sec:case-studies}

\subsection{Multitask Logistic Regression}
\label{sec:logistic-regression}

  To illustrate our framework, we analyze the performance of \adprep{} on logistic regression, as done by \citet{tripuraneni2020theory}.
  In this setting, we let $\theta = (B, w)$ for $B \in \reals^{d \times k}$ and $w \in \reals^k$, and define the predictor corresponding to $\theta$ to be $g_\theta(x) = x^\top Bw$.

  \subsubsection{Statistical Assumptions}

    As in the linear setting, we consider an input distribution $p$ with covariance $\Sigma$.
    We restrict the set of labels $\mathcal{Y}$ to $\set{0, 1}$, and consider the conditional distribution
    \[
      q(y \ | \ g_\theta(x)) = \bernoulli{\sigma(g_\theta(x))},
    \]
    where $\sigma(y) = 1/(1 + e^{-y})$ is the sigmoid function.

    We define the optimal parameters for tasks $t \in [T]$ to be $(B^\ast + \Delta_t^\ast, w_t^\ast)$, where $B^\ast$ is orthogonal and $\norm[F]{\Delta_t^\ast} \leq \delta_0$.
    As before, we define $\delta_t^\ast \defas \Delta_t^\ast w_t^\ast$ for any $t \in [T]$ and $W^\ast = [w_1, \dots, w_T] \in \reals^{k \times T}$.
    Having defined the prior quantities, we make use of the statistical assumptions presented in \secref{sec:lin-data-assumptions}, reproduced below for convenience:

    \begin{assumption}[Sub-Gaussian input]
      \label{assump:logistic-sg-inputs}
      There exists $\rho > 0$ such that if $x \sim p_t$, then $\Sigma^{-1/2}x$ is $\rho^2$-sub-Gaussian.
    \end{assumption}

    \begin{assumption}[Source task diversity]
      \label{assump:logistic-diversity-cond}
      For any $t \in [T]$, $\norm[2]{w_t^\ast} = \Theta(1) \leq r$, and $\sigma_k^2(W^\ast) = \Omega(r^2T/k)$.
    \end{assumption}

    Finally, we define the target task distribution $\rho$ by sampling $w^\ast, \delta^{\ast}$ uniformly from the $r$-- and $\delta_0$--balls of $\reals^k$ respectively, and letting $\theta^\ast = B^\ast w^\ast + \delta^\ast$.

  \subsubsection{Training Procedure}

    We train on the standard logistic loss:
    \[
      \ell(\hat{y}, y) = -y\log[\sigma(\hat y)] - (1 - y)\log[1 - \sigma(\hat y)].
    \]
    During source training, we optimize the representation over the set of orthogonal matrices while fixing $w_0 = 0$, \textit{i.e.} $\Theta_0 \defas \set{(B, 0) \suchthat \text{$B$ is orthogonal}}$.
    Let $B_0 \in \reals^{d \times k}$ be the obtained representation.
    To adapt to the target task, we initialize the learner at $\theta_0 = ([B_0, B_0], [w_0, -w_0])$ for some unit-norm vector $w_0$, and scale the predictor by a fixed parameter $\beta$ to be chosen later.
    Finally, we set the feasible sets for optimization to be
    \begin{align*}
      \constraintset_\source &\defas \set{(\Delta, w) \suchthat \norm[F]{\Delta} \leq \delta_0, \norm[2]{w} \leq r} \\
      \constraintset_\target^\beta &\defas \set{(\Delta, w) \suchthat \norm[F]{\Delta} \leq \delta_0/\beta, \norm[2]{w} \leq r\kappa^{1/2}/\beta},
    \end{align*}
    where $\kappa \defas \lambda_{\max}(\Sigma)/\lambda_{\min}(\Sigma)$.
    Note that this training procedure is quite similar to that of the linear setting.
    In particular, since the nonconvexity disappears in the limit $\beta \to \infty$, we can set $\beta$ appropriately as a function $n_\target$ so that the irreducible term arising from nonconvexity is negligible.

  \subsubsection{Performance Guarantee}

    Having described the statistical assumptions and the training procedure, we now specialize the guarantee of \thmref{thm:general-performance-bound} to this setting.
    Details are provided in \secref{sec:case-study-proofs}.
    \includestatement[dir=results/nonlinear/logistic]{logisticPerformanceBound}

\subsection{Two-Layer Neural Networks}
\label{sec:two-layer-nets}

  To further illustrate the framework, we instantiate the result in the two-layer neural network setting.
  Throughout this section, we fix an activation function $\sigma: \reals \to \reals$.

  \begin{assumption}
    \label{assump:nn-activation-assump}
    For any $t \in [-2, 2]$, $\abs{\sigma'(t)} \leq L$ and $\abs{\sigma''(t)} \leq \mu$.
    Furthermore, $\sigma(0) = 0$.
  \end{assumption}

  Then, for a constant $\beta$ and $\theta = (B, w)$, where $B \in \reals^{d \times 2k}$ and $w \in \reals^k$, we define the neural network $f^\beta_\theta(x) = \beta w^\top\sigma(B^\top x)$, where $\sigma$ is applied elementwise.

  Before we outline the statistical assumptions we make for this setting, we first illustrate a property of the function class.
  Consider any matrix $B_0 \in \reals^{d \times 2k}$ that can be expressed as $[A, A]$ for some $A \in \reals^{d \times k}$.
  Furthermore, fix $w_0 \in \set{-1, 1}^{2k}$ such that the last $k$ coordinates are a negation of the first $k$ coordinates.
  We will refer to parameters satisfying the previous conditions as being \emph{antisymmetric}.
  Note that for any antisymmetric parameter $\theta$, $f^\beta_{\theta} \equiv 0$.
  We can then define the following feature vectors:

  \begin{definition}[Feature vectors $\phi, \psi, \rho$]
    \label{def:feature-vectors}
    Let $\theta_0 = (B_0, w_0)$ be an antisymmetric parameter.
    Then, for every $x$, there exists feature vectors $\phi_{B_0}(x)$ and $\psi_{B_0, w_0}(x)$ such that
    \[
      f_{\theta_0 + (\Delta, w)}^{\beta}(x) = \beta w^\top\phi_{B_0}(x) + \beta\Delta^\top\psi_{B_0, w_0}(x) + \beta\zeta_{B_0, w_0}^{\Delta, w}(x).
    \]
    We interpret the features $\phi_{B_0}(x)$ and $\psi_{B_0, w_0}(x)$ to be the gradients of $(B, w) \mapsto f^\beta_{(B, w)}(x)$ evaluated at $\theta_0$\footnote{Closed-form expressions for these quantities are provided in \secref{sec:case-study-proofs}.}.
    Additionally, $\zeta_{B_0, w_0}^{\Delta, w}(x)$ is the Taylor error.
    Finally, we define $\rho_{B_0, w_0}(x)$ to be the concatenation of $\phi_{B_0}(x)$ and $\psi_{B_0, w_0}(x)$.
  \end{definition}

  We will show that if $\norm[F]{\Delta}$ and $\norm[F]{w}$ are both $O(1/\beta)$, then the remainder term $\zeta_{B_0, w_0}^{\Delta, w}(x)$ is $O(1/\beta)$, and thus the function class is approximately linear in $(\Delta, w)$ with feature functions $\phi_{B_0}(\cdot)$ and $\psi_{B_0, w_0}(\cdot)$.
  Note that these features correspond to the ``activation'' and ``gradient'' features, respectively, that are empirically evaluated by \citet{mu2020gradients}.

  \subsubsection{Statistical Assumptions}

    We now outline the statistical assumptions for this setting.
    For all tasks, the inputs are assumed to be sampled from a $1$-norm-bounded distribution $p$.
    Furthermore, we let $q(\cdot \ | \ \mu)$ be generated as $\mu + \eta$ for some $O(1)$-bounded additive noise $\eta$, similar to \citet{tripuraneni2020theory}.

    To define the source tasks, we fix a representation matrix $B^\ast$ and a linear predictor $w_0^\ast$ so that $\theta_0^\ast = (B^\ast, w_0^\ast)$ is antisymmetric, as motivated by the previous discussion.
    \begin{assumption}[Initialization Assumptions]
      \label{assump:nn-init-assump}
      The columns of $B^\ast$ are norm-bounded by $1$.
      Furthermore, $c_1I \preceq \expt{\rho_{\theta_0^\ast}(x)\rho_{\theta_0^\ast}(x)^\top} \preceq c_2I$ for $c_1 > 0$.
    \end{assumption}
    The assumption on the representation covariance ensures that that representation is well-conditioned; we define the condition number $\kappa = c_2/c_1$.
    We then define task-specific fine-tuning options by fixing unit-norm vectors $w_1^\ast, \dots, w_t^\ast \in \reals^{2k}$ and $\delta_1^\ast, \dots, \delta_t^\ast \in \reals^{k}$, as well as an orthonormal set of matrices $\Delta_1^\ast, \dots^\ast, \Delta_k^\ast \in \reals^{d \times 2k}$.
    Then, the optimal parameter for task $t$ is given by
    \[
      \theta_t^\ast = \left(B^\ast + \frac{1}{\beta}\sum_{i \in [k]}\delta_{t, i}^\ast\Delta_i^\ast, w_0^\ast + \frac{1}{\beta}w_t^\ast\right).
    \]
    We note that the fine-tuning step for $t \in [T]$ can be parametrized by $\omega_t^\ast = [w_t^\ast, \delta_t^\ast]/\beta \in \reals^{3k}$ via a linear transformation. We assume that this parametrization has a matrix representation $\Gamma$, so that $\theta_t^\ast = \theta_0^\ast + \Gamma\omega_t^\ast$.

    \begin{assumption}
      \label{assump:nn-diversity-condition}
      Let $\Omega \in \reals^{2k \times T}$ be the matrix $[\omega_1^\ast, \dots, \omega_T^\ast]$.
      Then, $\sigma_{2k}(\Omega) \gtrsim T/k$.
    \end{assumption}

    The assumption above is analogous to the diversity conditions assumed in the previous sections.

    Finally, to define the target task distribution, let $u, v \in \reals^k$ be i.i.d. samples from a uniform distribution over the unit sphere, and define $\theta^\ast = \theta_0^\ast + \Gamma[u, v]/\beta$.

  \subsubsection{Training Procedure}

    Finally, we describe the training procedure.
    Using squared-error loss, we train on the objective $\delta \mapsto \meanloss(f_{\theta_0 + \delta}^\gamma)$, where we set $\gamma = \beta$ during source training, while $\gamma$ is set as a function of $n_\target$ during target training time.
    Furthermore, we let $\Theta_0$ be the set of antisymmetric initializations satisfying \assumpref{assump:nn-init-assump}.
    Finally, we set the constraint sets
    \[
      \constraintset_\source = \set{(\Delta, w) \suchthat \norm[F]{\Delta}, \norm[2]{w} \leq 1/\beta} \quad \text{and} \quad \constraintset_\target^\gamma = \set{(\Delta, w) \suchthat \norm[F]{\Delta}^2 + \norm[2]{w}^2 \leq \kappa/\gamma^2}.
    \]

  \subsubsection{Performance Guarantee}

    Having described the statistical assumptions and the training procedure, we proceed with the performance guarantee.

    \includestatement[dir=results/nonlinear/nn]{nnPerformanceBound}

  \section{A \frzrep{} Hard Case for the Nonlinear Setting}
    \label{sec:gen-hard-case}
    In what follows, we establish the existence of a nonlinear setting where there exists a sample complexity separation between \adprep{} and \frzrep{}, as in \secref{sec:erm-hard-case}.
As before, fix $k, d \in \naturals$ with $2k < d$.
The construction relies on the observation that linear predictors lying in a rank-$k$ space are representable as linear functions of $2k$ appropriately chosen ReLU neurons.

Following the discussion in \secref{sec:two-layer-nets}, note that when we take $\beta \to \infty$, the resulting function class can be expressed as
\[
  f_{(B + \Delta, w)} = w^\top \sigma(B^\top x) + \inprod{x\sigma'(x^\top B), \Delta},
\]
where $B, \Delta \in \reals^{d \times 2k}$, and $w \in \reals^{2k}$.
We further constrain $B$ so that the first $k$ columns are equal to the negation of the last $k$ columns.
Finally, we choose $\sigma(x) = \max(x, 0)$, which we will also write as $x_+$ for convenience\footnote{Note that $\sigma'(x) = \ind{x > 0}$.}.
For convenience, we follow the convention in \secref{sec:two-layer-nets} of writing $\phi_{B}, \psi_{B}$ and $\rho_{B}$ for the activation, gradient, and concatenated features corresponding to $B$, respectively, as defined in \defref{def:feature-vectors}.

We briefly review the construction in \secref{sec:erm-hard-case}.
Let the input distribution $p$ be a Gaussian distribution on $\reals^d$ with covariance
\[
  \Sigma =
  \begin{bmatrix}
    \epsilon I_{d - k} & 0 \\
    0 & I_k \\
  \end{bmatrix}
\]
for a fixed $\epsilon \in (0, 1)$.
Define $E^\ast, E_k \subset \reals^d$ to be the eigenspaces corresponding to the first and second block of $\Sigma$, respectivcely.
For an orthogonal matrix $A \in \reals^{d \times k}$, define a distribution over $\theta$ given by
\begin{equation}
  \theta = \frac{1}{\sqrt{2\epsilon}}Av + \delta,
\end{equation}
where $v$ and $\delta$ are sampled uniformly at random from the unit spheres in $\reals^k$ and $E_k$, respectively.
As before, the family of interest is the set of distributions induced by any $A^\ast$ with $\col{A^\ast} \subseteq E^\ast$.

Now, we lift the linear task distribution setting into the ReLU setting.
In particular, we sample the source tasks by fixing orthogonal $A \in \reals^{d \times k}$, sampling $v$ and $\delta$ as before, and letting the optimal predictor be $f_{(B + \Delta, w)}$, where
\begin{equation}
  \label{eqn:relu-task-dist}
  B \defas [A, -A], \ w \defas \frac{1}{\sqrt{2\epsilon}}[v, -v], \ \Delta \defas \frac{1}{k}\mathbbm{1}\delta^\top.
\end{equation}
One can easily verify algebraically that
\[
  f_{(B + \Delta, w)}(x) = x^\top\left(\frac{1}{\sqrt{2\epsilon}}Av + \delta\right),
\]
as desired.
As before, we consider the family of task distributions induced by any $A^\ast$ with $\col{A^\ast} \subseteq E^\ast$.

With the above task distribution, we can then prove the following hardness result on \frzrep{}:

\includestatement[dir=results/nonlinear/hardcase]{ermReLUMinimaxBound}

In contrast, we have the following upper bound on the performance of \adprep{}:

\includestatement[dir=results/nonlinear/hardcase]{reLUAdaptTargPerf}

Proofs of the above results are provided in \secref{sec:relu-hard-case-proofs}.
As before, we compare the two methods when $n_\target = \Theta(d)$.
Then, from the results above, the lower bound on the loss of \frzrep{} is $\Omega(1)$, while the upper bound on the loss of \adprep{} is $O(1/\sqrt{n_\target})$.
Therefore, we also see a strict separation between the two methods within this setting as well, which grows with $n_\target \to \infty$.

  \section{Simulations}
    \label{sec:simulations}
    \begin{figure}[t]
  \centering
  \begin{minipage}{0.7\linewidth}
    \centering\captionsetup{width=0.9\linewidth,justification=centering}%
    \includegraphics[width=\linewidth]{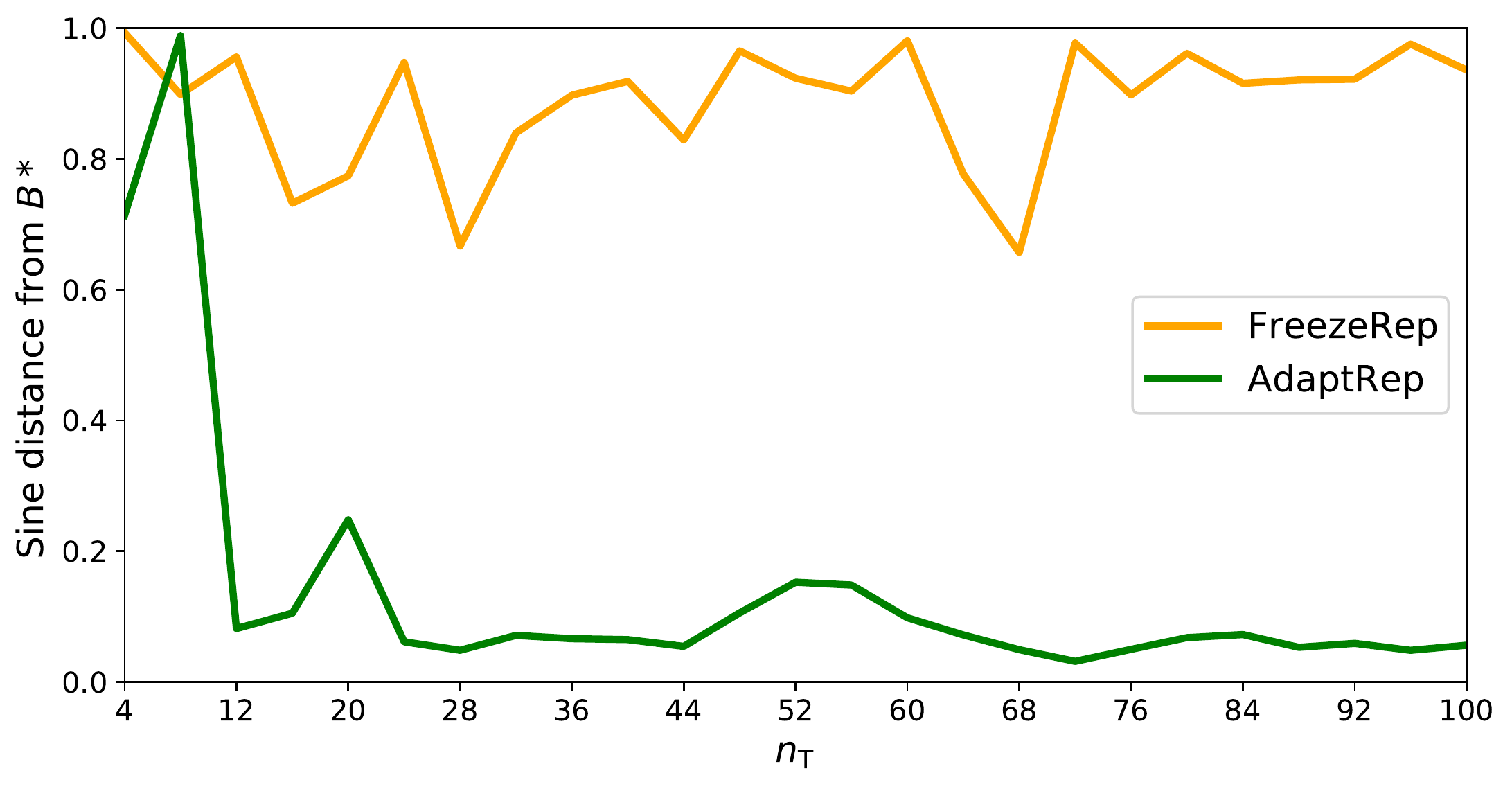}%
    \caption{Sine distances of the representations learned by each method from the correct space $B^\ast$.}
    \label{fig:sin-dists}
  \end{minipage}
  \hfill
  \begin{minipage}{0.7\linewidth}
    \centering\captionsetup{width=0.9\linewidth,justification=centering}%
    \includegraphics[width=\linewidth]{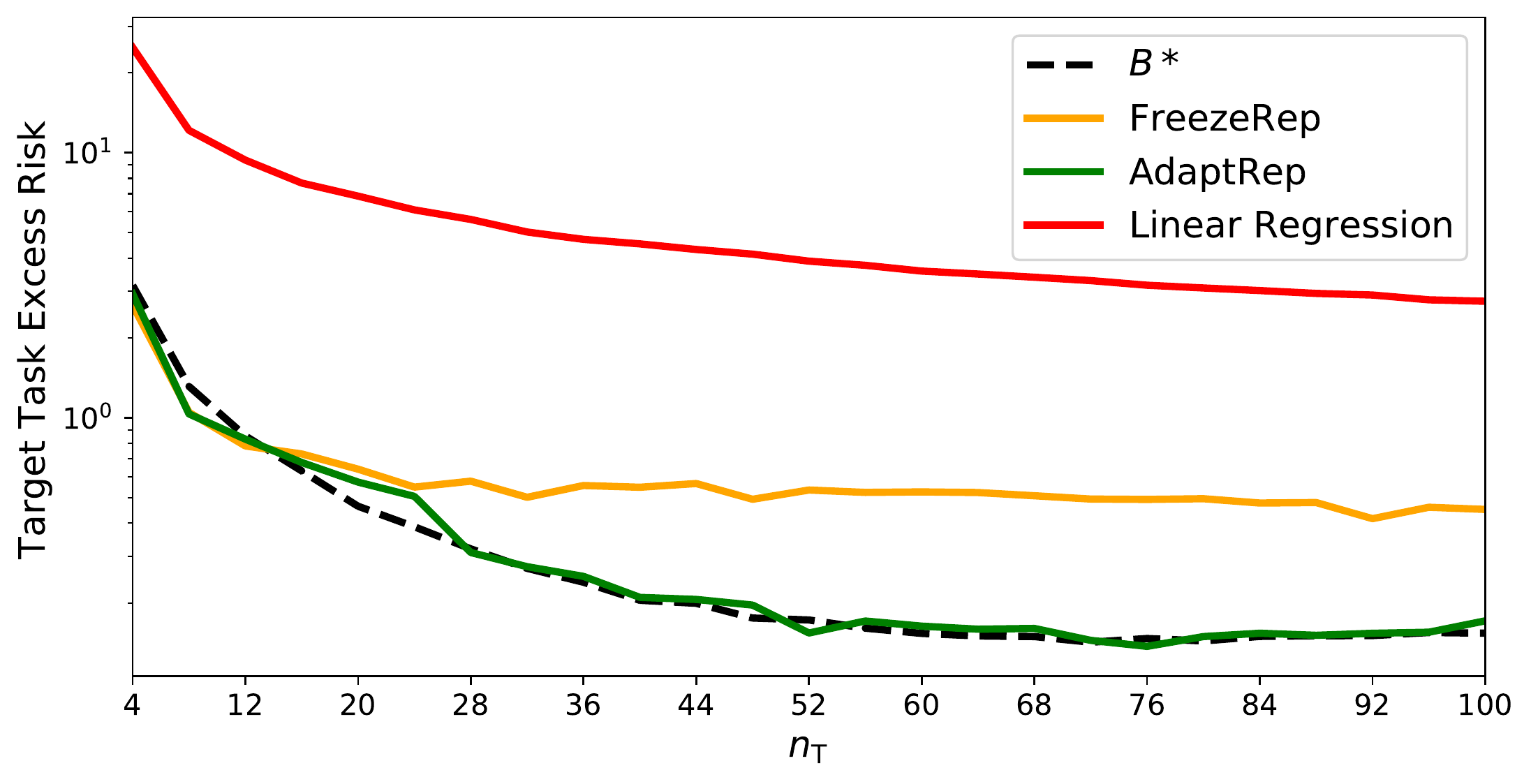}%
    \caption{The average worst-case excess risk for several settings of $n_\target$, plotted on a logarithmic $y$-axis.}
    \label{fig:excess-risks}
  \end{minipage}
\end{figure}

In this section, we experimentally verify the hard case for the linear setting presented in \secref{sec:erm-hard-case}. Since the empirical success of MAML or its variants in general has already been demonstrated extensively in practice and in existing work, it is not the focus of this section.

Recall that in the hard case, we vary the amount of data $n_\target$ available to the learner during target time, and we set $d \asymp n_\target$.
We fix $k = 2$ throughout the experiment.
The matrix $B^\ast$ spans the first $k$ coordinates, while the residuals $\delta_t^\ast$ lie in the span of the last $k$ coordinates.
Finally, we set the $\sigma = 2$ for the Gaussian noise.

During source training, both \frzrep{} and \adprep{} are provided with $1000$ tasks and $10d$ samples per task from the task distribution in \secref{sec:erm-hard-case}.
During target time, we evaluate the learned representation on the worst-case regression task from the same family.

Before we detail our results, we briefly comment on the nonconvexity in the source training procedure.
Rather than optimizing \eqnref{eqn:linear-source-procedure} or \eqnref{eqn:erm-lin-src-procedure} during source training, we use an additional Frobenius-norm regularizer on $B^\top B - WW^\top$ to ensure that the two terms are balanced.
In the case of \frzrep{}, this regularized objective was shown to have a favorable optimization landscape in \citet{tripuraneni2020provable}.
We then used L-BFGS to optimize these regularized objectives.
To further mitigate any possible optimization issues, we evaluated both methods with $10$ random restarts, and report the best of the $10$ restarts (as measured by the worst-case performance on the target task) for both methods.

\textbf{(Subspace Alignment).} First, we plot the alignment of the learned representation (using the best of the $10$ restarts described above) with the correct space $B^\ast$.
We measure this via the sine of the largest principal angle between the two spaces, i.e.
\[
  \sin \Theta_1(\hat{B}, B^\ast) = \sqrt{1 - \lambda_1^2\left(\hat{B}^\top B^\ast\right)}.
\]
We plot the results in \figref{fig:sin-dists}.
As predicted by \lemmaref{lemma:erm-learns-wrong-space}, \frzrep{} does not learn $B^\ast$, in contrast to \adprep{}.

\textbf{(Target Task Performance).} We proceeded to evaluate how the methods fare on their corresponding worst-case target tasks.
We do so by training with the representation over $1000$ i.i.d. draws of the target dataset, and averaging the excess risk over all obtained representations.
We provide the results in \figref{fig:excess-risks}, and include a comparison with standard linear regression as a baseline.
As predicted, \adprep{} performs much better than \frzrep{}, with a gap that grows with $n_\target$.

  \section{Conclusion}
    We have presented, to the best of our knowledge, the first statistical analysis of fine-tuning-based meta-learning.
We demonstrate the success of such algorithms under the assumption of approximately shared representation between available tasks.
In contrast, we show that methods analyzed by prior work that do not incorporate task-specific fine-tuning fail under this weaker assumption.

An interesting line of future work is to determine ways to formulate useful shared structure among MDPs, \emph{i.e.} formulate settings for which meta-reinforcement learning succeeds and results in improved regret bounds for downstream tasks.

  \section*{Acknowledgements}
    KC is supported by a National Science Foundation Graduate Research Fellowship, Grant DGE-2039656.
QL is supported by NSF \#2030859 and the Computing Research Association for the CIFellows Project.
JDL acknowledges support of the ARO under MURI Award W911NF-11-1-0303, the Sloan Research Fellowship, and NSF CCF 2002272.

  \bibliography{references.bib}
  \bibliographystyle{icml2021}

  \clearpage
  \appendix
  \onecolumn

  \section{Proof of \thmref{thm:lin-case-main-result}}
    \label{sec:lin-case-proof}
    In this section, we will prove the performance guarantee in the linear representation setting presented in \thmref{thm:lin-case-main-result}.
We first compute a bound on the difference in the spans of the true underlying representation $B^\ast$ and the representation $B_0$ obtained from training on the source tasks.
Having done so, we then analyze the performance of the best predictor found by projected gradient descent.

For clarity of presentation, we will write $\theta_t^\ast = B_t^\ast w_t^\ast + \delta_t^\ast$ and $\hat\theta_t = (B + \Delta_t)w_t$ throughout this section.
Furthermore, let $\hat\delta_t = \Delta_tw_t$.
Finally, we will be making use of the following covariance concentration results throughout this section, allowing us to connect empirical averages to population averages and vice versa:

\includestatement[dir=results/linear/perf,showproof=true]{linSrcCovConcentration}

\includestatement[dir=results/linear/perf,showproof=true]{linTargCovConcentration}

\subsection{Source Guarantees for the Linear Setting}

  We proceed to analyze the representation $B_0$ obtained from the source training procedure outlined in \eqnref{eqn:linear-source-procedure}.
  Key to the analysis is a bound on the average population loss over the source tasks that the global minimizer of \eqnref{eqn:linear-source-procedure} can achieve as a function of $n_\source$:

  \includestatement[dir=results/linear/perf,showproof=true]{linSrcTrainBound}

  The prior bound is central to the analysis, as the performance of the learner can be tied to how well $B_0$ spans the correct space.
  To see why this is the case, note that for large $n_\source$, the effect of the noise on the optimization in \eqnref{eqn:linear-source-procedure} is negligible.
  In this regime, no matter which representation $B_0$ the learner has chosen, the optimal predictor would satisfy $P_{X_tB_0}X_t\hat\theta_t \approx P_{X_tB_0}X_t\theta_t^\ast$ and $P_{X_tB_0}^\perp X_t\hat\theta_t \approx P_{X_tB_0}^\perp X_t\theta_t^\ast$.
  Consequently, the performance of the predictors chosen by the learner can be tied to the chosen representation $B_0$.
  We formalize this intuition in the following result:

  \includestatement[dir=results/linear/perf,showproof=true]{linTransferLemma}

\subsection{Target Guarantees for the Linear Setting}

  Having established a connection between the performance on the source tasks and the difference in the spans of $B^\ast$ and $B_0$, we can now analyze the performance of the target training procedure.
  First, we bound the performance of nearly optimal points in $\constraintset_\beta$ for several possible choices of $c_1, c_2$.

  \includestatement[dir=results/linear/perf,showproof=true]{linTargStatRates}

\subsection{Optimization Landscape during Target Time Training}

  Having derived statistical rates on nearly-optimal points for several choices of $\constraintset_\beta$ in the prior section, all that remains to be shown is that projected gradient descent can indeed find such points.
  In particular, we will demonstrate that for large enough $\beta$, the optimization landscape induced by $\meanloss_\beta$ is approximately convex.
  We do so by demonstrating that the objective satisfies the assumptions outlined in \secref{sec:pgd-performance-bound}, and thus the accompanying guarantees for projected gradient descent hold.

  \includestatement[dir=results/linear/perf,showproof=true]{linApproxLinearity}

  \includestatement[dir=results/linear/perf,showproof=true]{linLipLoss}

  \includestatement[dir=results/linear/perf,showproof=true]{linOptConstants}

\subsection{Deducing \thmref{thm:lin-case-main-result}}

  Having proven all the previous results, we can now assemble the main claim in \thmref{thm:lin-case-main-result}.
  Recall that we have defined the rates
  \begin{align*}
    r_\source(n_\source, T) &\defas \frac{\sigma^2}{n_\source T}\left(kT + kd\log\kappa n_\source + \log\frac{1}{\delta}\right) + \frac{\sigma\delta_0\norm[2]{\Sigma}^{1/2}}{\sqrt{n_\source T}}\sqrt{kT + kd\log\kappa n_\source + \log\frac{1}{\delta}} \\
      &\qquad + \frac{\sigma\delta_0}{\sqrt{n_\source}}\sqrt{\tr{\Sigma}\left(1 + \log\frac{T}{\delta}\right)} \\
    r_{\target}^{(1)}(n_\target) &\defas kr_{\source}(n_\source, T) + \frac{\sigma^2}{n_\target}\left(k + \log\frac{1}{\delta}\right) + \frac{\sigma\delta_0}{\sqrt{n_\target}}\sqrt{\tr\Sigma\left(1 + \log\frac{1}{\delta}\right)} \\
    r_{\target}^{(2)}(n_\target) &\defas \frac{\sigma^2}{n_\target}\left(k + \log\frac{1}{\delta}\right) \\
      &\qquad + \frac{\sigma}{\sqrt{n_\target}}\left[\delta_0 + \left(\frac{\norm[2]{w^\ast} + \delta_0\kappa^{1/2}}{\lambda_{\min}^{1/2}(\Sigma)}\right)\sqrt{kr_{\source}(n_\source, T)}\right]\sqrt{\tr\Sigma\left(1 + \log\frac{1}{\delta}\right)} \\
    r_{\target}^{(3)}(n_\target) &\defas \frac{\sigma^2}{n_\target}\left(d + \log\frac{1}{\delta}\right).
  \end{align*}

  \includestatement[dir=results/linear/perf,showproof=true]{linMainResult}

  \clearpage

  \section{Proofs for \secref{sec:erm-hard-case}}
    \label{sec:erm-hard-case-proofs}
    \includestatement[dir=results/linear/hardcase,showproof=true]{ermLearnsWrongSpace}

\includestatement[dir=results/linear/hardcase,showproof=true]{ermMinimaxBound}

  \clearpage

  \section{Proof of \thmref{thm:general-performance-bound}}
    \label{sec:gen-perf-proofs}
    In this section, we will prove the guarantee provided in \thmref{thm:general-performance-bound}, along with all related intermediate results.
Along these lines, we proceed as we did for \secref{sec:lin-case-proof}.
More explicitly, the overall outline of the proof follows the four major steps below.
We have placed in parenthesis the corresponding intermediate steps in the linear representation case (\secref{sec:lin-case-proof}) as an additional illustration of the procedure:

\begin{enumerate}[label=(\arabic*)]
  \item Provide a statistical rate for source training (\lemmaref{lemma:lin-source-train-bound}).

  \item Bound the difference in performance between the solution found by the optimization algorithm and the ERM solution (\lemmaref{lemma:lin-approx-linearity}+\thmref{thm:pgd-performance-bound}).

  \item Bound the difference between the ERM solution and the best solution in $\adaptset_{\theta_0}^{\constraintset_\target}$ as a function of target sample size (\lemmaref{lemma:lin-targ-stat-rates}).

  \item Connect the best-case performance in $\adaptset_{\theta_0}^{\constraintset_\target}$ to the performance of the learner on the source tasks (\Cref{lemma:lin-transfer-lemma,lemma:lin-targ-stat-rates,lemma:lin-opt-consts}).

\end{enumerate}

Note that step (4) is provided by the $(\nu, \epsilon)$-diversity condition.
We will now proceed to demonstrate the remaining steps.

\subsection{Bounding Average Source Task Performance (1)}

  The $(\nu, \epsilon)$-diversity condition implies that we can bound the best-case performance during target time training via control over the average source task performance.
  We proceed to provide such a bound using a standard uniform convergence argument.
  Recall that for a set of vector-valued functions $\mathcal{H}$ mapping from $\reals^{m} \to \reals^{n}$, the Rademacher complexity of $\mathcal{H}$ on $n_\source$ samples, denoted $\mathcal{R}_{n_\source}(\mathcal{H})$, is given by
  \[
    \mathcal{R}_{n_\source}(\mathcal{H}) \defas \expt{\sup_{h \in \mathcal{H}}\frac{1}{n_\source}\sum_{i \in [n_\source]}\sum_{j \in [n]}\epsilon_{ij}h(x_i)_j},
  \]
  where the expectation is over the samples $(x_i)$ and i.i.d. Rademacher random variables $\epsilon_{ij}$.

  \includestatement[showproof=true,dir=results/nonlinear/perf]{srcTrainBound}

\subsection{Bounding Optimization Performance (2)}

  Having provided a bound on source task performance, we now proceed to analyze the objective being optimized during target time training.
  We first note that the approximate linearity assumption in \assumpref{assump:adapt-approx-linearity} and the norm-boundedness of $\constraintset_\target$ via \assumpref{assump:target-adapt-norm-bound} ensure that the results from \secref{sec:pgd-performance-bound} apply, as long as we can show that the empirical loss satisfies an $(\alpha/\sqrt{n_\target})$-Lipschitz condition.
  We proceed to show that this is indeed a simple consequence of the $1$-Lipschitz assumption on the loss given by \assumpref{assump:loss-fn-assumptions}.

  \begin{lemma}
    Define the function $\meanloss: \reals^{n_\target} \to \reals$ as
    \[
      \meanloss(\hat{y}) \defas \frac{1}{n_\target}\sum_{i \in [n_\target]}\ell(\hat{y}_i, y_i).
    \]
    for fixed $y_1, \dots, y_{n_\target} \in \mathcal{Y}$.
    Then, for any $\hat{y}$, $\norm[2]{\grad[\hat{y}]\meanloss(\hat{y})}^2 \leq 1/n_\target$.
  \end{lemma}
  \begin{proof}
    By direct computation,
    \[
      \norm[2]{\grad[\hat y]\meanloss(\hat y)}^2 = \frac{1}{n_\target^2}\sum_{i \in n_\target}\abs{\grad[\hat y_i]\ell(\hat y_i, y)}^2 \leq \frac{1}{n_\target},
    \]
    where we have used the fact that $\ell(\cdot, y)$ is $1$-Lipschitz for any $y \in \mathcal{Y}$ by \assumpref{assump:loss-fn-assumptions}.
  \end{proof}

  Having verified that the assumptions in \secref{sec:pgd-performance-bound} hold, it follows that we have the following performance bound on projected gradient descent during target time training:

  \begin{lemma}
    \label{lemma:target-opt-perf}
    Assume that we run projected gradient descent during target training time for $T_\pgd$ iterations with step size $\eta$ given by
    \[
      \eta = \frac{1}{\sqrt{T_\pgd}}\left(\frac{R}{\sqrt{L^2 + \beta^2R^2}}\right).
    \]
    Let $g_0, \dots, g_{T_\pgd}$ denote the sequence of predictors obtained, where $g_0 = g_{\theta_0}$.
    Then, for any $g \in \adaptset_{\theta_0}^{\constraintset_\target}$,
    \[
      \min_{t}\meanloss^{g^\ast}(g_t) - \meanloss^{g^\ast}(g) \leq \beta R^2 + R\sqrt{\frac{L^2 + \beta^2 R^2}{T_\pgd}}.
    \]
  \end{lemma}

\subsection{Bounding the Performance of ERM during Target Training (3)}

  We now proceed to bound the performance of the ERM solution during target time training.
  Via following the canonical risk decomposition as in \lemmaref{lemma:nonlin-src-train-bound}, we can prove such a bound simply by bounding the maximum deviation between empirical and population losses over $\adaptset_{\theta_0}^{\constraintset_\target}$.

  \begin{lemma}
    \label{lemma:nonlin-targ-train-bound}
    Let $\mathcal{S}$ be the support of $\rho$.
    With probability at least $1 - \delta$ over the random draw of inputs and noise,
    \[
      \sup_{\substack{g^\ast \in \mathcal{S} \\ g \in \adaptset_{\theta_0}^{\constraintset_\target}}}\abs{\meanloss^{g^\ast}(g) - \meanloss^{g^\ast}_\infty(g)} \leq \frac{1}{\delta}\mathcal{R}_{n_\target}\left(\adaptset_{\theta_0}^{\constraintset_\target}\right) + \frac{B}{\delta\sqrt{n_\target}}.
    \]
  \end{lemma}
  \begin{proof}
    The proof proceeds similarly to that of \lemmaref{lemma:nonlin-src-train-bound}.
    Note that the supremum over $g^\ast$ does not affect the bound, since $g^\ast$ only enters into the expression through the labels $y_{i}$, and no matter what choice of $g^\ast$ is made, $\abs{\ell(0, y_i)} \leq B$ for all $i$.
  \end{proof}

\subsection{Concluding: Proving \thmref{thm:general-performance-bound}}

  Having completed all the steps for the outline, we now proceed to compile the main result.
  Intuitively, since the diversity condition allows us to bound the performance of the best predictor, and projected gradient descent can perform as well as any predictor in $\adaptset_{\theta_0}^{\constraintset_\target}$ (and thus, as well as the best predictor), we can obtain performance bounds on the iterates found while training on the target task.

  \includestatement[dir=results/nonlinear/perf,showproof=true]{genPerformanceBound}

  \clearpage

  \section{Proofs for Case Studies}
    \label{sec:case-study-proofs}
    \subsection{Logistic Regression}

  In this section, we verify that the logistic regression setting, as described in \secref{sec:logistic-regression}, satisfies the assumptions required by the general framework in \secref{sec:gen-setting-assumptions}.
  Subsequently, we then compute the necessary quantities for instantiating the guarantees provided by \thmref{thm:general-performance-bound} in this setting.

  \subsubsection{Verifying Assumptions of \secref{sec:gen-setting-assumptions}}

  It is easily verified that the logistic loss is $1$-Lipschitz, convex, and that $\abs{\ell(0, y)} \leq 1$ for $y \in \set{0, 1} = \mathcal{Y}$.
  Furthermore, as we have already characterized the approximate linearity of the function class in \lemmaref{lemma:lin-approx-linearity}, the approximate linearity assumption in \assumpref{assump:adapt-approx-linearity} holds with high probability.
  Finally, $\constraintset_\target$ is norm-bounded by $(\kappa r^2 + \delta_0^2)/\beta^2$.

  Therefore, all that remains is verifying that a $(\nu, \epsilon)$-diversity condition holds in this setting.
  In what follows, we will do so by connecting the logistic loss to squared error loss via leveraging smoothness and local strong convexity, as was done by \citet{tripuraneni2020theory}.
  Consequently, we can utilize the same argument as in the linear setting to obtain the desired diversity condition.

  \includestatement[showproof=true,dir=results/nonlinear/logistic]{logisticDiversityCondition}

  \subsubsection{Computations}

  Having demonstrated that the assumptions in \secref{sec:gen-setting-assumptions} holds, we proceed to calculate the relevant quantities required for establishing a performance bound in this setting.
  First, we compute the Rademacher complexity for source task training.

  \includestatement[showproof=true,dir=results/nonlinear/logistic]{logisticSourceRademacher}

  \includestatement[showproof=true,dir=results/nonlinear/logistic]{logisticTargetRademacher}

  \subsubsection{Compiling the Bound}

  \includestatement[showproof=true,dir=results/nonlinear/logistic]{logisticPerformanceBound}

\subsection{Two-Layer Neural Networks}

  Recall the feature vectors $\phi_{B_0}(x)$ and $\psi_{B_0, w_0}(x)$ defined in \defref{def:feature-vectors}, which have the following closed-form expressions:
  \[
  \phi_{B_0}(x) = \sigma(B_0^\top x) \quad \text{and} \quad \psi_{B_0, w_0}(x) = w_0^{\mathrm{diag}}\sigma'(B_0^\top x)x^\top.
  \]
  Additionally, there exists $\bar{B} = B_0 + \alpha_1 \Delta$ and $\bar{w} = w_0 + \alpha_2 w$ for some $\alpha_1, \alpha_2 \in [0, 1]$ such that the Taylor remainder $\zeta_{(B_0, w_0)}^{\Delta, w}(x)$ can be written as
  \[
    \zeta_{(B_0, w_0)}^{\Delta, w}(x) = \sum_{i=1}^{2k}\bar{w}_{i}\sigma''(\bar{B}_{i}^\top x)(\Delta_i^\top x)^2 + 2w_i\sigma'(\bar{B}_{i}^\top x)(\Delta_i^\top x).
  \]
  We define the feature vector $\rho_{B, w}(x)$ to be the concatentation of $\phi_B(x)$ and $\psi_{B, w}(x)$.
  For $\theta_0 = (B_0, w_0)$ and $\delta = (\Delta, w)$, we will frequently abuse notation and let $\rho_{\theta_0}(x)\delta$ denote the linearization of $f_{\theta_0 + \delta}$, \emph{i.e.}
  \[
    \rho_{\theta_0}(x)\delta = w^\top\phi_{B_0}(x) + \inprod{\psi_{\theta_0}(x), \Delta}.
  \]

  Before we verify the assumptions in \secref{sec:gen-setting-assumptions} required to instantiate the bounds, we first provide the following intermediate results on relevant norm bounds:

  \includestatement[dir=results/nonlinear/nn,showproof=true]{nnHessianBound}

  As an immediate corollary, since $\zeta_{\theta_0}^{\delta}(x)$ evaluates the Hessian at a point satisfying the preconditions of \lemmaref{lemma:nn-hessian-bound} during both source and target trainng time, we have the following result:

  \begin{corollary}
    \label{corollary:nn-remainder-bound}
    For any $\theta_0 \in \Theta_0$ and any $(\Delta, w)$ with $\norm[F]{\Delta} \leq c_1/\beta, \norm[2]{w} \leq c_2/\beta$, we have that $\abs{\beta\zeta_{\theta_0}^{(\Delta, w)}(x)} \leq c_1c_2(\mu + L)/\beta$ for $x \sim p$ almost surely.
  \end{corollary}

  The above result formalizes the intuition outlined in the main text, which claims that the function class is approximately linear in the fine-tuning parameters $(\Delta, w)$ for large enough $\beta$.
  Finally, we provide a norm bound on the combined representation vector $\rho_{\theta_0}(x)$.

  \includestatement[dir=results/nonlinear/nn,showproof=true]{nnRepresentationBound}

  \subsection{Verifying Assumptions of \secref{sec:gen-setting-assumptions}}

  Having proved the norm bounds above, we now can proceed to verify the required assumptions.
  First, we verify the assumptions on the loss function; clearly, the squared error loss is convex, so we simply need to verify that the loss is Lipschitz over the prediction, and that $\ell(0, y)$ is bounded.
  Note that we have to prove separate bounds for source and target training, due to the change in the function class.

  \includestatement[dir=results/nonlinear/nn,showproof=true]{nnLossAssumpValidation}

  \includestatement[dir=results/nonlinear/nn,showproof=true]{nnDiversityCondition}

  \includestatement[dir=results/nonlinear/nn,showproof=true]{nnApproximateLinearity}

  \subsubsection{Computations}

  \includestatement[dir=results/nonlinear/nn,showproof=true]{nnSourceRademacher}

  \includestatement[dir=results/nonlinear/nn,showproof=true]{nnTargetRademacher}

  \subsubsection{Compiling the Bound}

  \includestatement[dir=results/nonlinear/nn,showproof=true]{nnPerformanceBound}

  \clearpage

  \section{Proofs for the Nonlinear Hard Case}
    \label{sec:relu-hard-case-proofs}
    Throughout this section, we write $\theta_t^\ast = A^\ast v_t^\ast + \delta_t^\ast$ for the linear predictor parameter for source task $t \in [T]$.
Furthermore, we assume that this linear predictor corresponds to ReLU predictor $f_{(B^\ast + \Delta^\ast, w^\ast)}$.
First, we prove the following intermediate technical result which will be used throughout this section.

\includestatement[dir=results/nonlinear/hardcase,showproof=true]{optReLUPredictorIsLinear}

\subsection{Hardness Result for \frzrep{}}

\includestatement[dir=results/nonlinear/hardcase,showproof=true]{ermLearnsWrongNeurons}

\includestatement[dir=results/nonlinear/hardcase,showproof=true]{ermReLUMinimaxBound}

\subsection{Adaptation Upper Bound}

Having proven the minimax result for \frzrep{}, we now proceed to prove a corresponding upper bound on the performance of \adprep{}.
To do so, we need to prove a result analogous to \lemmaref{lemma:lin-transfer-lemma} for the ReLU setting.

Before we proceed to the proof, we need to find proper generalizations for relevant objects in the proof of \lemmaref{lemma:lin-transfer-lemma}.
In particular, we recall the prominent use of the projector $P_{\Sigma B_0}$, which, loosely speaking, can be thought of as representing the ``average'' component of the signal that can be represented by a parameter in $\col{B_0}$.

Recall that the inputs (which are element of $\reals^d$) are sampled from a distribution $p$.
This input distribution induces the $L_2(p)$-norm\footnote{we write $L_2$ throughout as a shorthand for $L_2(p)$.} on vector-valued functions of $\reals^d$ and its associated inner product via
\[
  \norm[L_2]{\zeta} = \expt{\norm[2]{\zeta(x)}^2} \quad \text{and} \quad \inprod{\zeta, \xi}_{L_2} = \expt{\zeta(x)^\top \xi(x)}.
\]
Now, let $\zeta: \reals^d \to \reals^p$ and $\xi: \reals^d \to \reals^q$ be two representation functions on $\reals^d$.
Then, we can define the \emph{linear projector onto $\zeta$} as the linear operator $P_\zeta$ taking representations $\reals^d \to \reals^q$ onto itself via
\[
  [P_\zeta \xi](x) \defas \expt{\xi(x)\zeta(x)^\top}\expt{\zeta(x)\zeta(x)^\top}^\dagger \zeta(x).
\]
We denote the corresponding orthogonal projection as $P_\zeta^\perp \xi \defas \xi - P_\zeta \xi$.
We note that these operators satisfy the orthogonality property
\[
  \expt{[P_\zeta \xi](x)^\top [P_\zeta^\perp \xi](x)} = 0,
\]
which can be easily verified algebraically.

To understand the operator $P_\zeta$, consider the problem of approximating a linear function of $\xi$ via a linear function of $\zeta$, as measured via the input distribution $p$.
More formally, for $v \in \reals^q$, we want to find
\[
  w^\ast = \argmin_{w \in \reals^p}\norm[L_2]{\xi(\cdot)^\top v - \zeta(\cdot)^\top w}^2 \quad \text{and} \quad f^\ast = \zeta(\cdot)^\top w^\ast.
\]
As the problem is differentiable and convex in $w$, we can simply use standard optimality conditions to find that
\[
  w^\ast = \expt{\zeta(x)\zeta(x)^\top}^\dagger\expt{\zeta(x)\xi(x)^\top}v \quad \text{and} \quad f^\ast = \zeta(\cdot)^\top w^\ast.
\]
That is, $P_\zeta\xi$ is performing exactly the transformation required on $\zeta$ such that $[P_\zeta\xi(\cdot)]^\top v$ is the best approximation to $\xi(\cdot)^\top v$ via linear functions of $\zeta$ in $L_2$-norm.
To further connect this construction to the linear setting, observe that if $\zeta(x) = B_0^\top x$ and $\xi(x) = x$, then for any $\theta, v \in \reals^d$,
\[
  \inprod{\xi(\cdot)^\top v, [P_\zeta^\perp\xi](\cdot)^\top\theta}_{L_2} = v^\top\Sigma^{1/2}P_{\Sigma B_0}^\perp\Sigma^{1/2}\theta,
\]
and thus $P_{\zeta}\xi$ is indeed the desired generalization of the projection operators used in the proof of \lemmaref{lemma:lin-transfer-lemma}.
Having introduced the required mathematical tools, we now prove the corresponding transfer lemma for this setting.

\includestatement[dir=results/nonlinear/hardcase,showproof=true]{reLUTransferLemma}

\includestatement[dir=results/nonlinear/hardcase,showproof=true]{reLUAdaptTargPerf}

  \clearpage

  \section{A Performance Bound for Projected Gradient Descent}
    \label{sec:pgd-performance-bound}
    In this section, we provide a performance bound for projected gradient descent on the objective
\[
  \meanloss(\theta) \defas \frac{1}{n}\sum_{i \in [n]}\ell(g_\theta(x_i), y_i)
\]
for $\theta \in \Theta$, where $\Theta \subseteq \reals^q$.
We assume that $\Theta$ is norm-bounded by $D$, and that $\Theta$ is convex and contains $0$.
Furthermore, we assume that $g_\theta$ is twice-differentiable as a function of $\theta$.

Key to the performance bound that we will demonstrate is that $g_\theta$ is ``approximately linear'' in the parameter $\theta$, which we formally define below.
Under this assumption, we demonstrate $\meanloss$ is approximately convex over $\Theta$ if $\meanloss$ is Lipschitz as a function of the vector of predictions and $\ell$ is convex in the first argument.
Therefore, with slight modifications to the online analysis of projected gradient descent, we obtain the desired performance bound.

Following the discussion above, we make the following assumptions:

\begin{assumption}[Approximate linearity]
  \label{assump:pgd-approximate-linearity}
  There exists $\beta$ and $L$ such that
  \[
    \sup_{\theta \in \Theta}\frac{1}{n}\sum_{i \in [n]}\norm[2]{\grad[\theta]^2g_\theta(x_i)}^2 \leq \beta^2 \quad \text{and} \quad \frac{1}{n}\sum_{i \in [n]}\norm[2]{\grad[\theta]g_{0}(x)}^2 \leq L^2.
  \]
\end{assumption}

\begin{assumption}[Assumptions on $\ell$]
  \label{assump:pgd-loss-assumptions}
  We assume that $\ell$ is convex in the first argument.
  Furthermore, if $\meanloss$ is viewed as a function of the vector of predictions $g_\theta(X)$, then we have that
  \[
    \norm[2]{\grad[g]\meanloss(g_\theta(X))}^2 \leq \frac{\alpha^2}{n}
  \]
  for any $\theta \in \Theta$, i.e. $\meanloss$ is $(\alpha/\sqrt{n})$-Lipschitz as a function of the vector of predictions.
\end{assumption}

Note that we abuse notation in \assumpref{assump:pgd-loss-assumptions}, using $\meanloss$ to reference both the function of the parameter, and of the vector of predictions.
Given these two assumptions, we now proceed to demonstrate that the loss landscape of $\meanloss$ has several desirable properties.

\includestatement[showproof=true,dir=results/optimization]{pgdApproximateConvexity}

\includestatement[showproof=true,dir=results/optimization]{pgdGradientBound}

Having proven the results above, we now proceed to the main claim of this section.

\includestatement[showproof=true,dir=results/optimization]{pgdPerformanceBound}

  \clearpage

  \section{Technical Lemmas}
    \begin{proposition}[\citet{du2020few}, Lemma A.5]
  \label{prop:orth-mat-covering}
  Let $\mathcal{O}^{d_1 \times d_2}$ be the set of matrices in $\reals^{d_1 \times d_2}$ with orthonormal columns, $d_1 \geq d_2$.
  Then, there exists an $\epsilon$-covering of $\mathcal{O}^{d_1 \times d_2}$ with at most $(6\sqrt{d_2}/\epsilon)^{d_1d_2}$ elements.
\end{proposition}

\begin{proposition}[Solving quadratic inequalities]
  \label{prop:solve-quad-ineqs}
  Assume that $ax^2 \leq bx + c$ for $a, b, c > 0$.
  Then, $bx + c \lesssim (b^2/a) + c$.
\end{proposition}
\begin{proof}
  Since $a > 0$, the solution set to the inequality is given by the interval $[r_1, r_2]$, where $r_1$ and $r_2$ are the roots of $ax^2 - bx - c$.
  By the quadratic formula, the larger root $r_2$ is given by
  \[
    r_2 = \frac{b + \sqrt{b^2 + 4ac}}{2a} \leq \frac{b}{a} + \sqrt{\frac{c}{a}}.
  \]
  Therefore,
  \[
    x \leq r_2 \leq \frac{b}{a} + \sqrt{\frac{c}{a}} \implies bx + c \leq \frac{b^2}{a} + \left(\frac{b}{\sqrt{a}}\right)\sqrt{c} + c \lesssim \frac{b^2}{a} + c,
  \]
  where the last inequality makes use of the Cauchy-Schwarz inequality.
\end{proof}

\begin{corollary}
  \label{cor:quadratic-expectation-lower-bound}
  Let $X$ and $Y$ be random variables.
  \[
    \expt{X^2} \lesssim \expt{(X + Y)^2} + \expt{Y^2}
  \]
\end{corollary}
\begin{proof}
  We have that
  \[
    \expt{(X + Y)^2} = \expt{X^2} + 2\expt{XY} + \expt{Y^2} \geq \expt{X^2} - 2\expt{\abs{XY}}.
  \]
  Therefore, by applying Cauchy-Schwarz,
  \[
    \expt{X^2} \leq \expt{(X + Y)^2} + 2\expt{\abs{XY}} \leq \expt{(X + Y)^2} + 2\sqrt{\expt{X^2}\expt{Y^2}}.
  \]
  Finally, by applying \propref{prop:solve-quad-ineqs},
  \[
    \expt{X^2} \lesssim \expt{(X + Y)^2} + \expt{Y^2}. \qedhere
  \]
\end{proof}


\begin{proposition}
  \label{prop:mat-to-proj}
  Let $A, B$ be matrices with compatible dimensions, and assume that $\rank{A} = r > 0$.
  Then,
  \[
    \norm[F]{P_A B} \leq \frac{1}{\sigma_r(A)}\norm[F]{A^\top B}.
  \]
\end{proposition}
\begin{proof}
  Let $(U, \Sigma, V^\top)$ be the compact singular value decomposition of $A$, \textit{i.e.} we only retain positive singular values in $\Sigma$.
  By rotational invariance,
  \[
    \norm[F]{A^\top B}^2 = \norm[F]{V\Sigma U^\top B}^2 = \norm[F]{\Sigma U^\top B}^2.
  \]
  Furthermore, by definition,
  \[
    \norm[F]{\Sigma U^\top B}^2 = \sum_{i}\norm[2]{\Sigma U^\top Be_i}^2 \geq \sigma_r^2(\Sigma)\sum_{i}\norm[2]{U^\top Be_i}^2 = \sigma_r^2(A)\norm[F]{U^\top B}^2.
  \]
  Finally, by applying rotational invariance once more,
  \[
    \norm[F]{A^\top B}^2 \geq \sigma_r^2(A)\norm[F]{UU^\top B}^2 = \sigma_r^2(A)\norm[F]{P_{A}B}^2,
  \]
  from which the desired claim follows.
\end{proof}

\begin{proposition}
  \label{prop:sum-regularizers-equivalence}
  Let $\lambda, \gamma > 0$, and fix a vector $y$.
  Then,
  \[
    \min_{\substack{A, x \\ Ax = y}}\frac{\lambda}{2}\norm[F]{A}^2 + \frac{\gamma}{2}\norm[2]{x}^2 = \sqrt{\lambda\gamma}\norm[2]{y}.
  \]
\end{proposition}
\begin{proof}
  We proceed by cases.
  If $y = 0$, then the result is trivial.

  Otherwise, if $y \neq 0$, note that $x^\ast \neq 0$.
  Now, for any fixed $x \neq 0$, the minimizing choice for $A$ is $zx^\top$ for some $z$.
  To see this, observe that if $A$ is not rank-$1$, then we can achieve a lower Frobenius norm by reducing its rank.
  Consequently, for a given $x$, the minimizing choice for $A$ is $yx^\top/\norm[2]{x}^2$ necessarily.
  Therefore,
  \[
    \min_{\substack{A, x \\ Ax = y}}\frac{\lambda}{2}\norm[F]{A}^2 + \frac{\gamma}{2}\norm[2]{x}^2 = \min_{x}\frac{\lambda}{2}\left(\frac{\norm[2]{y}^2}{\norm[2]{x}^2}\right) + \frac{\gamma}{2}\norm[2]{x}^2 = \min_{z > 0}\frac{\lambda}{2}\left(\frac{\norm[2]{y}^2}{z}\right) + \frac{\gamma z}{2}.
  \]
  This final optimization problem is convex in $z$ -- using first-order optimality conditions, we can thus easily see that $z^\ast = \sqrt{\lambda/\gamma}\norm[2]{y}$, and therefore
  \[
    \min_{\substack{A, x \\ Ax = y}}\frac{\lambda}{2}\norm[F]{A}^2 + \frac{\gamma}{2}\norm[2]{x}^2 = \sqrt{\lambda\gamma}\norm[2]{y}. \qedhere
  \]
\end{proof}

\begin{proposition}[Expectation bound on empirical spectral norm]
  \label{prop:empirical-spec-norm-exp-bound}
  Let $X \in \reals^{n \times d}$ be a matrix with rows drawn i.i.d. from a zero-mean distribution with covariance $\Sigma$.
  Furthermore, assume that the whitened distribution is $\rho^2$-sub-Gaussian.
  Then, whenever $n \gtrsim \rho^4d$,
  \[
    \expt{\lambda_{\max}\left(\frac{X^\top X}{n}\right)} \lesssim \norm[2]{\Sigma}.
  \]
\end{proposition}
\begin{proof}
  By Weyl's inequality,
  \[
    \expt{\lambda_{\max}\left(\frac{X^\top X}{n}\right)} \leq \norm[2]{\Sigma} + \expt{\abs{\lambda_{\max}\left(\frac{X^\top X}{n}\right) - \Sigma}} \leq \norm[2]{\Sigma} + \expt{\norm[2]{\frac{X^\top X}{n} - \Sigma}}.
  \]
  Thus, by applying the result in \citet[Theorem 4.4.1]{vershynin2017four}, we have that as long as $n_\source \gtrsim \rho^4 d$,
  \[
    \expt{\lambda_{\max}\left(\frac{X^\top X}{n}\right)} \lesssim \norm[2]{\Sigma}. \qedhere
  \]
\end{proof}

\begin{proposition}[Gaussian complexity chain rule, \citet{tripuraneni2020theory}, Theorem 7]
  \label{prop:gauss-comp-chain-rule}
  Assume that $\mathcal{F}$ is a class of functions $\reals^{k} \to \reals$ such that every $f \in \mathcal{F}$ is $L$-Lipschitz in the $L_2$-norm.
  Furthermore, assume that $\Phi$ is a class of functions $\reals^d \to \reals^k$ such that for any $\phi \in \Phi$, $\phi(x)$ is norm-bounded by $D$ for any $x$ in the support of the input distribution.
  Then, we have the bound
  \[
    \frac{1}{T}\mathcal{G}_{n}(\mathcal{F}^{\otimes T} \comp \Phi) \leq \frac{8D}{(nT)^2} + 128\left(\frac{L}{T}\mathcal{G}_{n}(\Phi) + \expt{\sup_{Z \in \mathcal{Z}}\mathcal{G}_Z(\mathcal{F})}\right)\log(nT),
  \]
  where $\mathcal{Z}$ is the random set $\set{(\phi(x_{i_1}), \dots, \phi(x_{i_n})) \suchthat i_1, \dots, i_n \in [nT]}$ and $\mathcal{G}_Z(\mathcal{F})$ is the empirical Gaussian complexity on samples $Z$.
  Note that the inner expectation is over the $nT$ input samples, and that we have assumed that all input samples come from a single distribution.
\end{proposition}

\begin{proposition}[\citet{tripuraneni2020theory}, Lemma 6]
  \label{prop:opt-perf-under-subopt-repr}
  Let $h, h^\ast: \reals^d \to \reals^k$ be representation functions, and define
  \[
    \Lambda(h, h^\ast) \defas \expt{h^\ast(x)h^\ast(x)^\top} - \expt{h^\ast(x)h(x)^\top}\left(\expt{h(x)h(x)^\top}\right)^\dagger\expt{h(x)h^\ast(x)^\top}.
  \]
  Then, $\inf_{v}\expt{(h(x)^\top v - h^\ast(x)^\top v^\ast)^2} = (v^\ast)^\top\Lambda(h, h^\ast)v^\ast$.
  Furthermore, if
  \[
    \sigma_{\min}(\expt{h(x)h(x)^\top}) \geq c_1 > 0 \quad \text{and} \quad \sigma_{\max}(\expt{h^\ast(x)h^\ast(x)^\top}) \leq c_2,
  \]
  then this infimum is achieved within the ball of radius $\norm[2]{v^\ast}\sqrt{c_2/c_1}$.
\end{proposition}
\begin{proof}
  The calculation of the infimum is provided in \citet{tripuraneni2020theory}, and is thus omitted.
  However, we prove the sharper radius bound below.

  Define $F_{h,h'} \defas \expt{h(x)h'(x)^\top}$, so that $\Lambda(h, h^\ast) = F_{h^\ast,h^\ast} - F_{h^\ast,h}F_{h, h}^\dagger F_{h, h^\ast}$.
  Then, since $\Lambda(h, h^\ast) \succeq 0$, and recalling that the infimum is achieved at $v = F_{h, h}^\dagger F_{h, h^\ast}v^\ast$,
  \[
    \norm[2]{F_{h, h}^\dagger F_{h, h^\ast}v^\ast}^2 \leq \frac{1}{c_1}\norm[2]{F_{h, h}^{1/2}F_{h, h}^\dagger F_{h, h^\ast}v^\ast}^2 \leq \frac{1}{c_1}\norm[2]{F_{h^\ast, h^\ast}^{1/2}v}^2 \leq \frac{c_1}{c_2}\norm[2]{v^\ast}^2,
  \]
  as desired.
\end{proof}

\end{document}